\documentclass{article}

\usepackage[T1]{fontenc}    
\usepackage{graphicx}
\usepackage{microtype}      
\usepackage{algorithm2e}
\usepackage{psfrag}
\usepackage{url}
\usepackage{xcolor}
\definecolor{dgreen}{rgb}{0.0,0.545,0.0}

\graphicspath{{figures/}}
\usepackage{bm}
\usepackage{amsthm}
\usepackage{amssymb}
\usepackage{mathtools}
\usepackage{threeparttable}
\newtheorem{theorem}{Theorem}[section]
\theoremstyle{definition}

\title{Symplectic Neural Networks in Taylor Series Form for Hamiltonian Systems}

\author{Yunjin Tong$^a$\thanks{Co-first author},~~~~Shiying Xiong$^a$\thanks{Co-first author,~~shiying.xiong@dartmouth.edu},~~~~Xingzhe He$^{a,b}$,\\
Guanghan Pan$^{a,c}$,~~~~Bo Zhu$^a$\\
$a$ Dartmouth College, Hanover, NH 03755, United States\\
$b$ Rutgers University, New Brunswick, NJ 08854, United States\\
$c$ Middlebury College, Middlebury, VT 05753, United States
}

\begin{document}
\maketitle
\begin{abstract}
We propose an effective and light-weight learning algorithm, Symplectic Taylor Neural Networks (Taylor-nets), to conduct continuous, long-term predictions of a complex Hamiltonian dynamic system based on sparse, short-term observations.
At the heart of our algorithm is a novel neural network architecture consisting of two sub-networks. Both are embedded with terms in the form of Taylor series expansion designed with symmetric structure. The key mechanism underpinning our infrastructure is the strong expressiveness and special symmetric property of the Taylor series expansion, which naturally accommodate the numerical fitting process of the gradients of the Hamiltonian with respect to the generalized coordinates as well as preserve its symplectic structure. We further incorporate a fourth-order symplectic integrator in conjunction with neural ODEs' framework into our Taylor-net architecture to learn the continuous-time evolution of the target systems while simultaneously preserving their symplectic structures. We demonstrated the efficacy of our Taylor-net in predicting a broad spectrum of Hamiltonian dynamic systems, including the pendulum, the Lotka--Volterra, the Kepler, and the H\'enon--Heiles systems. Our model exhibits unique computational merits by outperforming previous methods to a great extent regarding the prediction accuracy, the convergence rate, and the robustness despite using extremely small training data with a short training period (6000 times shorter than the predicting period), small sample sizes, and no intermediate data to train the networks. \\

\end{abstract}

\section{Introduction}
Hamiltonian mechanics, first formulated by William Rowan Hamilton in 1834 \cite{Hamilton1834}, is one of the most fundamental mathematical tools for analyzing the long-term behavior of complex physical systems studied over the past centuries \cite{Viswanath1994,Feng2010}. Hamiltonian systems are ubiquitous in nature, exhibiting total energy with various forms, as seen in plasma physics \cite{Morrison2005}, electromagnetic physics \cite{Li2019}, fluid mechanics \cite{Salmon1988}, and celestial mechanics \cite{Saari1996}. 
Mathematically, Hamiltonian dynamics describe a physical system by a set of canonical coordinates, i.e., generalized positions and generalized momentum, and uses the conserved form of the symplectic gradient to drive the temporal evolution of these canonical coordinates \cite{Hand2008}. However, for a dynamic system governed by some unknown mechanics, it is challenging to identify the Hamiltonian quantity and its corresponding symplectic gradients by directly observing the system's state, especially when such observation is partial and the sample data is sparse \cite{grigo2019,brunton2016,stinis2019}.

The rapid advent of machine learning (ML) techniques opens up new possibilities to solve the identification problems of physical systems by statistically exploring their underlying structures. On the one hand, data-driven approaches have proven their efficacy in uncovering the underlying governing equations of a variety of physical systems, encompassing applications in fluid mechanics \cite{Brunton2020}, wave physics \cite{Hughes2019}, quantum physics \cite{Sellier2019}, thermodynamics \cite{hernandez2020}, and material science \cite{Teicherta2019}. 
On the other hand, various ML methods have been proposed to boost the numerical simulation of complex dynamical systems by incorporating learning paradigms into simulation infrastructures, e.g., ordinary differential equations \cite{regazzoni2019}, linear or nonlinear partial differential equations \cite{xiong2020roenets,regazzoni2019,Raissi2018,Pang2019,holiday2019,rudy2017}, high-dimensional partial differential equations \cite{Sirignano2018}, inverse problems \cite{Raissi2019}, space-fractional
differential equations \cite{Gulian2018}, systems with noisy multi-fidelity data \cite{Raissi2017}, and pseudo-differential operators \cite{feliu2020,fan2019}, to name a few. More recently, many lines of research have tried to incorporate physical priors into the learning framework, instead of letting the learning algorithm start from scratch, e.g., embedding the notion of an incompressible fluid \cite{mohan2020,xiong2020neural}, the Galilean invariance \cite{ling2016}, a quasistatic
physics simulation \cite{geng2020}, and the invariant quantities in Lagrangian systems \cite{cranmer2020} and Hamiltonian systems \cite{hernandez2020,Greydanus2019,Jin2020,Zhong2020Symplectic,dipietro2020sparse,xiong2020nonseparable}.

There are two critical aspects in learning and predicting the dynamics of a Hamiltonian system.
The first key point is to learn the continuous dynamic time evolution. It is impossible to control the growth of approximation error and monitor the level of error by simply using neural networks to learn the dynamics of a system and integrating using traditional integrators, e.g., Euler \cite{Hairer1987}, Runge--Kutta \cite{Runge1895, Kutta1901}. Secondly and also more challengingly, finding the symplectic gradients that have symmetric structure is hard. The exact solution of a Hamiltonian system leads to a symplectic map from the initial conditions to an arbitrary present state. Due to inaccuracies arising from the computed gradients of a high-dimensional Hamiltonian using traditional neural networks, finding the exact structure of the symplectic gradients from non-differentiable functions will often cause a large error.
To address these two critical aspects, we propose the following solutions. Firstly, we utilize the neural ODE (ODE-net)'s framework, introduced by Chen et al. in 2018, \cite{chen2018neural}, to obtain the continuous evolution. Drawing parallels between residual neural networks \cite{He2016} and the modeling pattern of an ODE, Chen et al. utilize continuously-defined dynamics to naturally incorporate data
that arrive at arbitrary times. The main difficulty lies in addressing the second aspect. To preserve symplectic structure while accurately approximating the continuous-time evolution of dynamical systems, the neural networks have to fulfill two criteria:

\begin{enumerate}
    \item The gradients of the Hamiltonian with respect to the generalized coordinates should be symmetric.
    \item The temporal integration should be symplectic.
\end{enumerate}

We made two essential contributions to meet the above two criteria when processing a Hamiltonian system by incorporating a set of special computing primitives into traditional neural networks. First, to enable symmetric gradients of the Hamiltonian with respect to the generalized coordinates, we construct neural networks that model the gradients and preserve their symmetric structure. Due to the multi-nonlinear-layer architecture of traditional deep neural networks, it is impossible for these networks to fulfill the symmetric property. Thus, we can only use a three-layer network with the form of \emph{linear-activation-linear}, where the weights of the two linear layers are the transpose of each other. However, such a shallow network cannot capture the complexity of Hamiltonian systems. Therefore, in order to maintain the expressive power of the network, we create multiple such three-layer sub-networks and combine them linearly into the Taylor series form. As a result, our network architecture naturally preserves the symmetry of the structure while exhibiting strong expressive power. Furthermore, to enable a symplectic preserving temporal evolution, we implement a fourth-order symplectic integrator \cite{Forest1990, zhu2020deep} within a neural ODE-net architecture \cite{chen2018neural, zhu2020inverse}. This fourth-order integration step enables an explicit fourth-order symplectic mapping to preserve the canonical character of the equations of motion in an exact manner. In other words, it preserves the property that the temporal evolution of a Hamiltonian system yields a canonical transformation from the initial conditions to the final state \cite{Forest1990}.

Based on these two major enhancements, we propose a novel neural network model, \emph{symplectic Taylor neural networks (Taylor-nets)}, to precisely preserve the quantity and predict the dynamics of a Hamiltonian system. The Taylor-nets consist of two sub-networks whose outputs are combined using a fourth-order symplectic integrator. Both sub-networks are embedded with the form of Taylor series expansion and learn gradients of the position and momentum of the Hamiltonian system, respectively. We design the sub-networks such that each term of the Taylor series expansion is symmetric. The symmetric property of the terms and the fourth-order symplectic integrator ensure our model intrinsically preserves the symplectic structure of the underlying system. Therefore, the prediction made by our neural networks leads to a symplectic map from an initial condition to the present state of a Hamiltonian system, which is the most fundamental feature of the exact solution of a Hamiltonian system.

With the integrated design of the sub-networks symmetric structure and the fourth-order symplectic integrator, our learning algorithm is capable of utilizing extremely limited training data to generate highly accurate predictive results that satisfy the conservation laws in various forms. In particular, we demonstrate that the training period of our model can be around 6000 times shorter than its predicting period (other methods have the training period 1--25 times shorter than the predicting period \cite{chen2018neural, Greydanus2019, Jin2020}), and the number of training samples is around 5 times smaller (meaning we use 5 times fewer time-sequences as in the training process) than that used by other methods. Moreover, our method only requires the data collected at the two endpoints of the training period to train the neural networks, without requiring any intermediate data samples in between the initial point and the endpoint. These improvements are crucial for modeling a realistic, complex physical system because they minimize the requirement of training data, which are typically difficult to obtain, and reduce training time by a significant amount. Other major computational merits of our proposed method include its fast convergence rate and robustness. Thanks to the intrinsic structure-preserving characteristic of our method, our model converges more than 10 times faster than the other methods and is more robust under large noise.
Overall, the contributions of our work can be summarized as below:

\begin{itemize}

\item We design a neural network architecture that intrinsically preserves the symplectic structure of the underlying system and predicts the continuous-time evolution of a Hamiltonian system.

\item We embed the form of Taylor series expansion into the neural networks with each term of the Taylor series expansion designed to be symmetric.

\item Our model outperforms other state-of-the-art methods regarding the prediction accuracy, the convergence rate, and the robustness despite using small data with a short training period, small sample sizes, and no intermediate data to train the model.
\end{itemize}

Our work is inspired by previous methodologies that incorporate the symplectic structure of a Hamiltonian system into neural networks. Greydanus et al. first tried to enforce conservative features of the Hamiltonian system by reformulating the loss function using Hamilton's equations, known as Hamiltonian neural networks (HNNs) \cite{Greydanus2019}. Based on HNNs, many works were developed. Chen et al. developed symplectic recurrent neural networks (SRNN), which is a recurrent HNN that relies on a symplectic integrator \cite{Chen2019}. Toth et al. developed the Hamiltonian Generative Network (HGN), learning Hamiltonian dynamics from high-dimensional observations (such as images) without restrictive domain assumptions \cite{Toth2020}. Zhong introduced Symplectic ODE-Net (SymODEN), which adds an external control term to the standard Hamiltonian dynamics in order to learn the system dynamics which conform to Hamiltonian dynamics with control \cite{Zhong2020Symplectic}. Methods like HNN, which focuses on the reformulation of the loss function, incur two main limitations. On the one hand, it requires the temporal derivatives of the momentum and the position of the systems to calculate the loss function, which is difficult to obtain from real-world systems. On the other hand, HNN doesn't strictly preserve the symplectic structure, because its symplectomorphism is realized by its loss function rather than its intrinsic network architecture. Our model successfully bypasses the time derivatives of the datasets by incorporating an integrator solver into the network architecture. Moreover, we design our model differently by embedding a symmetric structure into the neural networks, instead of manipulating the loss function. Thus, our model can strictly preserve the symplectic structure.

Independently, an intrinsic way to encode the symplectic structure is introduced by Jin et al. \cite{Jin2020}. Such neural networks are called Symplectic networks (SympNets), which intrinsically preserve the symplectic structure for identifying Hamiltonian systems. Motivated by SympNets, we invent a neural network architecture to intrinsically preserve the symplectic structure. However, our model preserves two major advantages over SympNets. First, our model is capable of learning the continuous-time evolution of dynamical systems. Second, our model can easily be extended to N-body systems. The parameters scale in the matrix map for training $N$ dimensional Hamiltonian system of our model is $O(1)$. The number of parameters does not increase since based on the interactive models between particle pairs we only need data collected from two bodies as the training data to predict the dynamics of many bodies. However, SympNets require $O(N^2)$ complexity, which makes it hard to generalize to the high-dimensional N-body problems.

The structure of this paper is as follows. In section \ref{sec:math}, we will first introduce the mathematical formulas and their proofs that serve as the foundation of our methodology. Then, we will discuss the design of our neural networks in Taylor series form as well as the proofs of their symplectic structure-preserving property. The next section \ref{sec:num_meth} describes the implementation details and numerical results, which compare our methodology with other state-of-the-art methods, such as ODE-net and HNN. In section \ref{sec:nbody}, we extend the application of our methodology to solve an N-body problem. Lastly, conclusions are drawn in a section \ref{sec:conclusions} with discussions of potential directions of our future research.

\section{Mathematical foundation}\label{sec:math}

\subsection{Hamiltonian mechanics}\label{subsec:symp_int}

We start by considering a Hamiltonian system with $N$ pairs of canonical coordinates (i.e. $N$ generalized positions and $N$  generalized momentum). The time evolution of canonical coordinates is governed by the symplectic gradient of the Hamiltonian \cite{Hand2008}.
Specifically, the time evolution of the system is governed by Hamilton's equations as

\begin{equation}
\begin{dcases}
\frac{\textrm{d} \bm{q}}{\textrm{d} t} = \frac{\partial \mathcal {H}}{\partial \bm{p}} ,\\
\frac{\textrm{d} \bm{p}}{\textrm{d} t} =-\frac{\partial \mathcal {H}}{\partial \bm{q}},
\end{dcases}
\label{eq:Hamilton}
\end{equation}

with the initial condition

\begin{equation}
(\bm{q}(t_0),\bm{p}(t_0)) = (\bm q_0,\bm p_0).
\label{eq:intH}
\end{equation}

In a general setting, $\bm{q}=(q_1,q_2,\cdots,q_N)$ represents the positions and $\bm{p}=(p_1,p_2,...p_N)$ denotes their momentum. Function $\mathcal H = \mathcal H(\bm q, \bm p)$ is the Hamiltonian, which corresponds to the total energy of the system.
By assuming that the Hamiltonian is separable, we can rewrite the Hamiltonian in the form

\begin{equation}
   \mathcal {H}(\bm{q}, \bm{p})= T(\bm{p}) + V(\bm{q}).
   \label{eq:Hpq}
 \end{equation}
This happens frequently in Hamiltonian mechanics, with $T$ being the kinetic energy and $V$ the potential energy. Substituting \eqref{eq:Hpq} into \eqref{eq:Hamilton} yields

\begin{equation}
\begin{dcases}
\frac{\textrm{d} \bm{q}}{\textrm{d} t} = \frac{\partial T(\bm p)}{\partial \bm{p}},\\
\frac{\textrm{d} \bm{p}}{\textrm{d} t} =-\frac{\partial V(\bm q)}{\partial \bm{q}}.
\end{dcases}
\label{eq:HpqVT}
\end{equation}
This set of equations is fundamental in designing our neural networks. Our model will learn the right-hand side (r.h.s.) of \eqref{eq:HpqVT} under the framework of ODE-net.

One of the important features of the time evolution of Hamilton's equations is symplectomorphism, which represents a transformation of phase space that is volume-preserving.
In the setting of canonical coordinates, symplectomorphism means the transformation of the phase flow of a Hamiltonian system conserves the symplectic two-form

\begin{equation}
    \textrm{d} \bm p\wedge \textrm{d} \bm q \equiv \sum_{j=1}^{N}\left(\textrm{d}p_j\wedge \textrm{d}q_j\right),
    \label{eq:dpq}
\end{equation}
where $\wedge$ denotes the wedge product of two differential forms. Inspired by the symplectomorphism feature, we aim to construct a neural network architecture that intrinsically preserves Hamiltonian structure.

\subsection{A symmetric network in Taylor expansion form}
\label{subsec:taylor}

In order to learn the gradients of the Hamiltonian with respect to the generalized coordinates, we propose the following underpinning mechanism, which is a set of symmetric networks that learn the gradients of the Hamiltonian with respect to the generalized coordinates.

\begin{equation}
\begin{dcases}
\bm T_p(\bm p,\bm \theta_p) \rightarrow \frac{\partial T(\bm p)}{\partial \bm p},\\
\bm V_q(\bm q,\bm \theta_q) \rightarrow \frac{\partial V(\bm q)}{\partial \bm q},
\end{dcases}
\label{eq:TpVq}
\end{equation}
with parameters $(\bm\theta_p,\bm\theta_q)$ that are designed to learn the r.h.s. of \eqref{eq:HpqVT}, respectively. Here, the ``$\rightarrow$" represents our attempt to use the left-hand side (l.h.s) to learn the r.h.s.
Substituting \eqref{eq:TpVq} into \eqref{eq:HpqVT} yields

\begin{equation}
\begin{dcases}
\frac{\textrm{d} \bm{q}}{\textrm{d} t} = \bm T_p(\bm p,\bm \theta_p),\\
\frac{\textrm{d} \bm{p}}{\textrm{d} t} = -\bm V_q(\bm q,\bm \theta_q).
\end{dcases}
\label{eq:HpqVT1}
\end{equation}
Therefore, under the initial condition \eqref{eq:intH}, the trajectories of the canonical coordinates can be integrated as

\begin{equation}
\begin{dcases}
\bm q(t) = \bm q_0 + \int_{t_0}^{t} \bm T_p(\bm p,\bm \theta_p) \textrm{d}t,\\
\bm p(t) = \bm p_0 - \int_{t_0}^{t} \bm V_q(\bm q,\bm \theta_q) \textrm{d}t.
\end{dcases}
\label{eq:TVint}
\end{equation}

From \eqref{eq:TpVq}, we obtain

\begin{equation}
\begin{dcases}
\frac{\partial \bm T_p(\bm p,\bm \theta_p)}{\partial \bm p} \rightarrow \frac{\partial^2 T(\bm p)}{\partial \bm p^2},\\
\frac{\partial \bm V_q(\bm q,\bm \theta_q)}{\partial \bm q} \rightarrow \frac{\partial^2 V(\bm q)}{\partial \bm q^2}.
\end{dcases}
\label{eq:dTpVq}
\end{equation}
The r.h.s of \eqref{eq:dTpVq} are the Hessian matrix of $T$ and $V$ respectively, so we can design $\bm T_p(\bm p,\bm \theta_p)$ and $\bm V_q(\bm q,\bm \theta_q)$ as symmetric mappings, that are

\begin{equation}
\frac{\partial \bm T_p(\bm p,\bm \theta_p)}{\partial \bm p} = \left[\frac{\partial \bm T_p(\bm p,\bm \theta_p)}{\partial \bm p}\right]^T,
\label{eq:partial_T}
\end{equation}
and

\begin{equation}
\frac{\partial \bm V_q(\bm q,\bm \theta_q)}{\partial \bm q} = \left[\frac{\partial \bm V_q(\bm q,\bm \theta_q)}{\partial \bm q}\right]^T.
\label{eq:partial_V}
\end{equation}

Due to the multiple nonlinear layers in the construction of traditional deep neural networks, it is impossible for these deep neural networks to fulfill \eqref{eq:partial_T} and \eqref{eq:partial_V}. Therefore, we can only use a three-layer network with the form of \emph{linear-activation-linear}, where the weights of the two linear layers are the transpose of each other, and in order to still maintain the expressive power of the networks, we construct symmetric nonlinear terms, as same as the terms of a Taylor polynomial, and combine them linearly.
Specifically, we construct a symmetric network $\bm T_p(\bm p,\bm \theta_p)$ as

\begin{equation}
   \bm T_p(\bm p,\bm \theta_p) =\left( \sum_{i = 1}^{M}\bm A_i^T \circ f_i \circ \bm A_i - \bm B_i^T \circ f_i \circ \bm B_i\right) \circ \bm p + \bm b,
   \label{eq:Tp_Taylor}
\end{equation}
where `$\circ$' denotes the function composition, $\bm A_i$ and $\bm B_i$ are fully connected layers with size $N_h\times N$, $\bm b$ is a $N$ dimensional bias, $M$ is the number of terms in the Taylor series expansion, and $f_i$ is an element-wise function, representing the $i^{\textrm{th}}$ order term in the Taylor polynomial

\begin{equation}
f_i(x) = \frac{1}{i!}x^i.
\label{eq:Taylor_ex}
\end{equation}
Figure \ref{fig:Taylor_net} plots a schematic diagram of $\bm T_p(\bm p,\bm \theta_p)$ in Taylor-net. The input of $\bm T_p(\bm p,\bm \theta_p)$ is $\bm p$, and $\bm \theta_p = (\bm A_i$, $\bm B_i, \bm b )$. We construct a negative term $\bm B_i^T \circ f_i \circ \bm B_i$ following a positive term $\bm A_i^T \circ f_i \circ \bm A_i$, since two positive semidefinite matrices with opposite signs can represent any symmetric matrix.

\begin{figure}
  \centering
  \includegraphics[width=.7\linewidth]{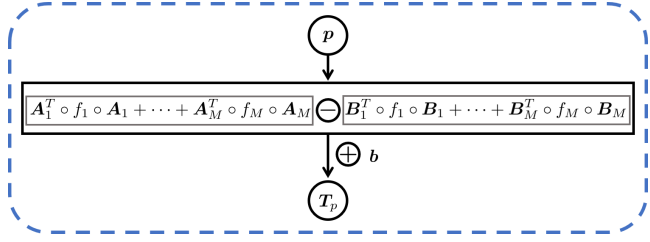}\\[0.5mm]
  \caption{The schematic diagram of $ \bm T_p(\bm p,\bm \theta_p)$ in Taylor-net.}
  \label{fig:Taylor_net}
\end{figure}

To prove \eqref{eq:Tp_Taylor} is symmetric, that is it fulfills \eqref{eq:partial_T}, we introduce theorem \ref{thm:Tp_sym}.
\begin{theorem}
The network \eqref{eq:Tp_Taylor} satisfies \eqref{eq:partial_T}.
\label{thm:Tp_sym}
\end{theorem}
\begin{proof}
From \eqref{eq:Tp_Taylor}, we have

\begin{equation}
    \frac{\partial \bm T_p(\bm p,\bm \theta_p)}{\partial \bm p} = \sum_{i = 1}^{M}\bm A_i^T \bm \Lambda_i^A \bm A_i - \bm B_i^T \bm \Lambda_i^B\bm B_i,
    \label{eq:Tp_d}
\end{equation}

with

\begin{equation}
    \Lambda_i^A = \textrm{diag}\left(\frac{\textrm{d}f}{\textrm{d}x}\Bigg|_{x= \bm A_i\circ \bm p}\right),
\end{equation}

and

\begin{equation}
    \Lambda_i^B = \textrm{diag}\left(\frac{\textrm{d}f}{\textrm{d}x}\Bigg|_{x= \bm B_i\circ \bm p}\right).
\end{equation}

It's easy to see that \eqref{eq:Tp_d} is a symmetric matrix that satisfies \eqref{eq:partial_T}.
\end{proof}
In fact, $\bm T_p(\bm p,\bm \theta_p)$ in \eqref{eq:partial_T} and $\bm V_q(\bm q,\bm \theta_q)$ in \eqref{eq:partial_V} satisfy the same property, so we construct $V_q$ with the similar form as

\begin{equation}
   \bm V_q(\bm q,\bm \theta_q) =\left( \sum_{i = 1}^{M}\bm C_i^T \circ f_i \circ \bm C_i - \bm D_i^T \circ f_i \circ \bm D_i\right) \circ \bm q + \bm d.
   \label{eq:Vq_Taylor}
\end{equation}
Here, $\bm C_i$, $\bm D_i$, and $\bm d$ have the same structure as \eqref{eq:Tp_Taylor}, and $(\bm C_i$, $\bm D_i, \bm d )= \bm \theta_q$.

\subsection{Symplectic Taylor neural networks}\label{subsec:Sym_ode}
Next, we substitute the constructed network \eqref{eq:Tp_Taylor} and \eqref{eq:Vq_Taylor} into \eqref{eq:TVint} to learn the Hamiltonian system \eqref{eq:HpqVT}.
We employ ODE-net \cite{chen2018neural} as our computational infrastructure.
Here we briefly introduce the essential idea of ODE-net for completeness.
Under the perspective of viewing a neural network as a dynamic system, we can treat the chain of residual blocks in a neural network as the solution of an ODE with the Euler method. Given a residual network that consists of sequence of transformations
\begin{equation}
  \bm{h}_{t+1}=\bm{h}_{t}+f(\bm{h}_{t}, \theta_t),
  \label{eq:neural1}
\end{equation}
the idea is to parameterize the continuous dynamics using an ODE specified by a neural network:
\begin{equation}
  \frac{\textrm{d} \bm{h}(t)}{\textrm{d} t} = f(\bm{h}_{t}, t, \theta).
  \label{eq:neural2}
\end{equation}

\RestyleAlgo{ruled}
\begin{algorithm}
  \caption{Integrate \eqref{eq:TVint} by using the fourth-order symplectic integrator}
  \label{alg:int_net}
  \SetAlgoLined
  \KwIn{$\bm q_0,\bm p_0,t_0,t,\Delta t$,\\
  $\bm F_t^j$ in \eqref{eq:Ft} and $\bm F_k^j$ in \eqref{eq:Fk} with $j=1,2,3,4$;}
  \KwOut{$\bm q(t),\bm p(t)$}
  $n = \textrm{floor}[(t-t_0)/\Delta t]$\;
  \For{$i = 1,n$}{
  $(\bm k_p^0,\bm k_q^0) = (\bm p_{i-1},\bm q_{i-1})$;\\
  \For{$j = 1, 4$}{
  $(\bm t_p^{j-1}, \bm t_q^{j-1}) =\bm F_t^j(\bm k_p^{j-1},\bm k_q^{j-1},\Delta t)$;\\
  $(\bm k_p^j, \bm k_q^j) =\bm F_k^j(\bm t_p^{j-1},\bm t_q^{j-1},\Delta t)$;\\}
  $(\bm p_{i},\bm q_{i}) = (\bm k_p^4,\bm k_q^4)$;}
  $\bm q(t)=\bm q_{n},\bm p(t) = \bm p_{n}$.
\end{algorithm}

Inspired by the idea of ODE-net, we design neural networks that can learn continuous time evolution. In Hamiltonian system \eqref{eq:HpqVT}, where the coordinates are integrated as \eqref{eq:TVint}, we can implement a time integrator to solve for $\bm p$ and $\bm q$. While ODE-net uses fourth-order Runge--Kutta method to make the neural networks structure-preserving, we need to implement an integrator that is symplectic. Therefore, we introduce Taylor-net, in which we design the symmetric Taylor series expansion and utilize the fourth-order symplectic integrator to construct neural networks that are symplectic to learn the gradients of the Hamiltonian with respect to the generalized coordinates and ultimately the temporal integral of a Hamiltonian system.

For the constructed networks \eqref{eq:Tp_Taylor} and \eqref{eq:Vq_Taylor}, we integrate \eqref{eq:TVint} by using the fourth-order symplectic integrator \cite{Forest1990}. Specifically, we will have an input layer $(\bm q_0,\bm p_0)$ at $t = t_0$ and an output layer $(\bm q_n,\bm p_n)$ at $t = t_0 + n \textrm{d} t$. The recursive relations of $(\bm q_i,\bm p_i), i = 1,2,\cdots,n$, can be expressed by the algorithm \ref{alg:int_net}. The input function in algorithm \ref{alg:int_net} are

\begin{equation}
\bm F_t^j(\bm p,\bm q,\textrm{d}t) =  \left(\bm p,\bm q+ c_j\bm T_p(\bm p,\bm \theta_p)\textrm{d}t\right),
\label{eq:Ft}
\end{equation}
and
\begin{equation}
\bm F_k^j(\bm p,\bm q,\textrm{d}t) =  \left(\bm p -  d_j \bm V_q(\bm q,\bm \theta_q)\textrm{d}t,\bm q\right),
\label{eq:Fk}
\end{equation}
with
\begin{equation}
\begin{aligned}
c_{1}&=c_{4}={\frac {1}{2(2-2^{1/3})}},&c_{2}&=c_{3}={\frac {1-2^{1/3}}{2(2-2^{1/3})}},&\\
d_{1}&=d_{3}={\frac {1}{2-2^{1/3}}},&d_{2}&=-{\frac {2^{1/3}}{2-2^{1/3}}},& d_{4}=0.
\end{aligned}
\end{equation}
The derivation of the coefficients $c_{j}$ and $d_{j}$ can be found in \cite{Forest1990,Yoshida1990,Candy1991}. Relationships \eqref{eq:Ft} and \eqref{eq:Fk} are obtained by replacing $\partial T(\bm p)/\partial \bm p$ and $ \partial V(\bm q)/\partial \bm q$ in the fourth-order symplectic integrator with deliberately designed neural networks $\bm T_p(\bm p,\bm \theta_p)$ and $\bm V_q(\bm q,\bm \theta_q)$, respectively. Figure \ref{fig:Sym_Taylor_net} plots a schematic diagram of Taylor-net which is described by algorithm \ref{alg:int_net}. The input of Taylor-net is $(\bm q_0,\bm p_0)$, and the output is $(\bm q_n,\bm p_n)$. Taylor-net consists of $n$ iterations of fourth-order symplectic integrator. The input of the integrator is $(\bm q_{i-1},\bm p_{i-1})$, and the output is $(\bm q_{i},\bm p_{i})$. Within the integrator, the output of $\bm T_p$ is used to calculate $\bm q$, while the output of $\bm V_q$ is used to calculate $\bm p$, which is signified by the shoelace-like pattern in the diagram. The four intermediate variables $\bm t_p^0\cdots \bm t_p^4$ and $\bm k_q^0\cdots \bm k_q^4$ indicate that the scheme is fourth-order.

\begin{figure}
  \centering
  \includegraphics[width=0.93\linewidth]{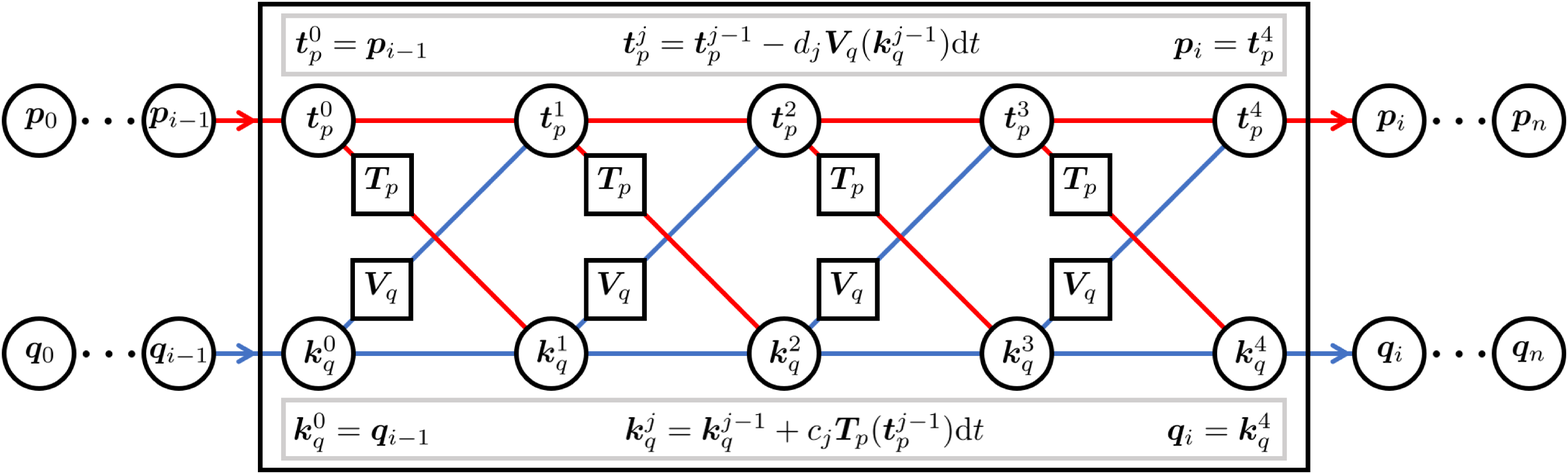}\\[0.5mm]
  \caption{The schematic diagram of Taylor-net. The input of Taylor-net is $(\bm q_0,\bm p_0)$, and the output is $(\bm q_n,\bm p_n)$. Taylor-net consists of $n$ iterations of fourth-order symplectic integrator. The input of the integrator is $(\bm q_{i-1},\bm p_{i-1})$, and the output is $(\bm q_{i},\bm p_{i})$. The four intermediate variables $\bm t_p^0\cdots \bm t_p^4$ and $\bm k_q^0\cdots \bm k_q^4$ show that the scheme is fourth-order.}
  \label{fig:Sym_Taylor_net}
\end{figure}

By constructing the network $\bm T_p(\bm p,\bm \theta_p)$ in \eqref{eq:Tp_Taylor} that satisfies \eqref{eq:partial_T}, we show that theorem \ref{thm:sym_Ft} holds, so the network \eqref{eq:Ft} preserves the symplectic structure of the system.

\begin{theorem}
For a given $\textrm{d}t$, the mapping $\bm F_t^j(:,:,\textrm{d}t):\mathbb{R}^{2N}\rightarrow \mathbb{R}^{2N}$ in \eqref{eq:Ft} is a symplectomorphism if and only if the Jacobian of $\bm T_p$ is a symmetric matrix, that is, it satisifies \eqref{eq:partial_T}.
 \label{thm:sym_Ft}
\end{theorem}

\begin{proof}
Let

\begin{equation}
(\bm t_p, \bm t_q) = \bm F_t^j(\bm k_p, \bm k_q,\textrm{d}t).
\end{equation}

From \eqref{eq:Ft}, we have

\begin{equation}
\begin{aligned}
&\textrm{d}\bm t_p \wedge \textrm{d}\bm t_q = \textrm{d}\bm k_p \wedge \textrm{d}\bm k_q + \\
&\frac{1}{2}\sum_{l,m=1}^N c_j\textrm{d}t  \left[\frac{\partial \bm T_p(\bm k_p,\bm \theta_p)}{\partial \bm k_p}\Bigg|_{l,m} - \frac{\partial \bm T_p(\bm k_p,\bm \theta_p)}{\partial \bm k_p}\Bigg|_{m,l}\right]\textrm{d}\bm k_p|_l \wedge \textrm{d}\bm k_q|_m.
\label{eq:dtk}
\end{aligned}
\end{equation}

Here $\bm A|_{l,m}$ refers to the entry in the $l$-th row and $m$-th column of a matrix $\bm A$, $\bm x|_l$ refers to the $l$-th component of vector $\bm x$. From \eqref{eq:dtk}, we know that $\textrm{d}\bm t_p \wedge \textrm{d}\bm t_q = \textrm{d}\bm k_p \wedge \textrm{d}\bm k_q$ is equivalent to

\begin{equation}
    \frac{\partial \bm T_p(\bm k_p,\bm \theta_p)}{\partial \bm k_p}\Bigg|_{l,m} - \frac{\partial \bm T_p(\bm k_p,\bm \theta_p)}{\partial \bm k_p}\Bigg|_{m,l}=0,\quad \forall l,m= 1,2,\cdots,N,
\end{equation}
which is \eqref{eq:partial_T}.
\end{proof}
Similar to the theorem\ref{thm:sym_Ft}, we can find the relationship between $\bm F_k^j$ and the Jacobian of $\bm V_q$. The proof of \ref{thm:sym_Fk} is omitted as it is similar to the proof of the theorem\ref{thm:sym_Ft}.
\begin{theorem}
For a given \textrm{d}t, the mapping $\bm F_k^j(:,:,\textrm{d}t):\mathbb{R}^{2N}\rightarrow \mathbb{R}^{2N}$ in \eqref{eq:Fk} is a symplectomorphism if and only if the Jacobian of $\bm V_q$ is a symmetric matrix, that is, it satisifies \eqref{eq:partial_V}.
\label{thm:sym_Fk}
\end{theorem}

Suppose that $\Phi_1$ and $\Phi_2$ are two symplectomorphisms. Then, it is easy to show that their composite map $\Phi_2\circ \Phi_1$ is also symplectomorphism due to the chain rule. Thus, the symplectomorphism of the algorithm \ref{alg:int_net}can be guaranteed by the theorems \ref{thm:sym_Ft} and \ref{thm:sym_Fk}.

\section{Numerical methods and results}\label{sec:num_meth}
This section discusses the details of our implementation, including the numerical method to generate training data, the construction of the neural networks, and the predictions for arbitrary time points on a continuous timeline.

\subsection{Dataset Generation}\label{subsec:data}
To make a fair comparison with the ground truth, we generate our training and testing datasets by using the same numerical integrator based on a given analytical Hamiltonian. In the learning process, we generate $N_{train}$ training samples, and for each training sample, we first pick a random initial point $(\bm{q}_0,\bm{p}_0)$ (input), then use the symplectic integrator discussed in section \ref{subsec:symp_int} to calculate the value $(\bm{q}_n,\bm{p}_n)$ (target) of the trajectory at the end of the training period $T_{train}$. We do the same to generate a validation dataset with $N_{validation}=100$ samples and the same time span as $T_{train}$ and calculate the validation loss $L_{validation}$ along the training loss $L_{train}$ to evaluate the training process. In addition, we generate a set of testing data with $N_{test}=100$ samples and predicting time span $T_{predict}$ that is around 6000 times larger and calculate the prediction error $\epsilon_p$ to evaluate the predictive ability of the model.
For simplicity, we use $(\bm{\hat{p}}_n,\bm{\hat{q}}_n)$ to represent the predicted values using our trained model.

We remark that our training dataset is relatively smaller than that used by the other methods. Most of the methods, e.g. ODE-net \cite{chen2018neural} and HNN \cite{Greydanus2019}, have to rely on intermediate data in their training data to train the model. That is the dataset is $[(\bm{q}_{0}^{(s)}, \bm{p}_{0}^{(s)}),(\bm{q}_{1}^{(s)}, \bm{p}_{1}^{(s)}),\dots,(\bm{q}_{n-1}^{(s)},\bm{p}_{n-1}^{(s)}), (\bm{q}_n^{(s)}, \bm{p}_n^{(s)})]_{s=1}^{N_{train}}$ , where $(\bm{q}_{1}^{(s)},\bm{p}_{1}^{(s)})\dots,(\bm{q}_{n-1}^{(s)},\bm{p}_{n-1}^{(s)})$ are $n-1$ intermediate points collected within $T_{train}$ in between $(\bm{q}_{0}^{(s)}, \bm{p}_{0}^{(s)})$ and $(\bm{q}_n^{(s)},  \bm{p}_n^{(s)})$. On the other hand, we only use two data points per sample, the initial data point and the end point, and our dataset looks like $\left[(\bm{q}_{0}^{(s)},\bm{p}_{0}^{(s)}), (\bm{q}_n^{(s)},\bm{p}_n^{(s)})\right]_{s=1}^{N_{train}}$, which is $n-1$ times smaller the dataset of the other methods, if we do not count $(\bm{q}_{0}^{(s)},\bm{p}_{0}^{(s)})$.
Our predicting time span $T_{predict}$ is around 6000 times the training period used in the training dataset $T_{train}$ (as compared to 10 times in HNN). This leads to a 600 times compression of the training data, in the dimension of temporal evolution. Note that we fix $T_{train}$ and $T_{predict}$ in practice so that we can train our network more efficiently on GPU. One can also choose to generate training data with different $T_{train}$ for each sample to obtain more robust performance.

\subsection{Test Cases}
We consider the pendulum, the Lotka--Volterra, the Kepler, and the H\'enon--Heiles systems in our implementation.
\paragraph{Pendulum system}
The Hamiltonian of an ideal pendulum system is given by

\begin{equation}
   \mathcal {H}(q, p) =\frac{1}{2}p^2-\cos{(q)}.
   \label{eq:pendu}
 \end{equation}
We pick a random initial point for training $(\bm{q}_0,\bm{p}_0)\in \left[-2,2\right]\times\left[-2,2\right]$.

\paragraph{Lotka--Volterra system}
For a Lotka--Volterra system, its Hamiltonian is given by

  \begin{equation}
   \mathcal {H}(q, p) =p-e^{p}+2q-e^{q}.
   \label{eq:LV}
 \end{equation}
Similarly, we pick a random initial point for training $(\bm{q}_0,\bm{p}_0)\in \left[-2,2\right]\times\left[-2,2\right]$.

\paragraph{Kepler system}
Now we consider a eight-dimensional system, a two-body problem in 2-dimensional space. Its Hamiltonian is given by

   \begin{equation}
   \begin{aligned}
   \mathcal {H}(\bm{q}, \bm{p}) = &\mathcal {H}(q_1,q_2,q_3,q_4,p_1,p_2,p_3,p_4)\\
   =& \frac{1}{2}(p_1^2+ p_2^2+p_3^2+ p_4^2)-\frac{1}{\sqrt{q_1^2+ q_2^2+q_3^2+ q_4^2}},
   \label{eq:kep1}
   \end{aligned}
 \end{equation}
where $(q_1,q_2)$ and $(p_1,p_2)$ are the position and momentum associated with the first body, $(q_1,q_2)$ and $(p_3,p_4)$ are the position and momentum associated with the second body.
We randomly pick the initial training point $(\bm{q}_0,\bm{p}_0)\in \left[-3,3\right]\times\left[-2,2\right]$, and enforce a constraint on the initial $(q_1,q_2)$ and $(p_1,p_2)$ so that they are at least separated by some distance $L_d=4$. This is to avoid having infinite force immediately.

\paragraph{H\'enon--Heiles system}

Lastly, we introduce a four-dimensional H\'enon--Heiles system, which is a non-integrable system. This kind of chaotic system is generally hard to model. Its Hamiltonian is defined as

 \begin{equation}
   \mathcal {H}(\bm{q}, \bm{p}) = \mathcal  {H}(q_1,q_2, p_1,p_2)= \frac{1}{2}
   (p_1^2+ p_2^2)+\frac{1}{2}(q_1^2+ q_2^2)+(q_1^2 q_2-\frac{q_2^3}{3}),
   \label{eq:hh1}
 \end{equation}
The random initial point for training is $(\bm{q}_0,\bm{p}_0)\in \left[-0.5,0.5\right]\times\left[-0.5,0.5\right]$.

\subsection{Training settings and ablation tests}

For all four systems, we use the Adam optimizer \cite{kingma2014adam}. We choose the automatic differentiation method as our backward propagation method. We have tried both the adjoint sensitivity method, which is used in ODE-net \cite{chen2018neural} and the automatic differentiation method. Both methods can be used to train the model well. However, we found that using the adjoint sensitivity method is much slower than using the automatic differentiation method considering the large parameter size of neural networks. Therefore, we use the automatic differentiation method in our implementation. The detailed derivation of adjoints formulas under the setting of Taylor-net and the prediction result can be found in \ref{sec:adjoint}.

All $A_i$ and $B_i$ in \eqref{eq:Tp_Taylor} are initialized as $A_i, B_i \sim \mathcal{N}(0,\sqrt{2/[N*N_h*(i+1)]})$, where $N$ is the dimension of the system and $N_h$ is the size of the hidden layers. The loss function is

\begin{equation}
L_{train}=\frac{1}{N_{train}}\sum_{s=1}^{N_{train}}\|\bm{\hat{p}}_n^{(s)}-\bm{p}_n^{(s)}\|_1+\|\bm{\hat{q}}_n^{(s)}-\bm{q}_n^{(s)}\|_1.
\label{eq:loss}
\end{equation}
The validation loss $L_{validation}$ is the same as \eqref{eq:loss} but with dataset different from the training dataset. We choose $L1$ loss, instead of Mean Square Error (MSE) loss because $L1$ loss performs better in all cases given in Table \ref{tab:problems}. We conduct the ablation test on these problems to compare the validation loss after convergence with different training loss functions in the training process. Figure \ref{fig:L1_comparison} shows the comparison of validation losses with different training loss functions in the training process of different problems validated by L1 loss function. Figure \ref{fig:MSE_comparison} shows the comparison of validation losses with different training loss functions in the training process of different problems validated by MSE loss function. We observe that for all problems, the validation loss with $L1$ is smaller than that with MSE after convergence. We believe the better performance of $L1$ may be due to MSE loss's high sensitivity to outliers. Hence, we choose to use $L1$ loss as our training loss function.

\begin{table}
  \caption{Set-up of problems.}
  \centering
  \setlength{\tabcolsep}{1mm}{
  \begin{tabular}{lcccc}
  \hline
  Problems & Pendulum & Lotka-Volterra & Kepler & H\'enon--Heiles\\
  \hline
  Hamiltonian & \eqref{eq:pendu}&\eqref{eq:LV}&\eqref{eq:kep1}&\eqref{eq:hh1}\\
  $T_{train}$ &0.01 &0.01 & 0.01&0.01\\
  $T_{predict}$ & $20\pi$&$20\pi$&$20\pi$&10\\
  $N_{train}$ & 15 &25&25&25\\
  Epoch & 100& 150 & 50&100\\
  Learning rate & 0.002 & 0.003 & 0.001&0.001\\
  $step\_size$ & 10 & 10 & 10 & 10\\
  $\gamma$ & 0.8 & 0.8 & 0.8&0.8\\
  $M$ & 8 & 8 & 20&12\\
  $N_h$ & 16 & 8 & 8 & 16\\
  $L_{train}$& $2.75\times 10^{-5}$ & $2.37\times 10^{-5}$ & $7.29\times 10^{-5}$ &$9.24\times 10^{-6}$\\
  $L_{validation}$& $1.39\times 10^{-4}$ & $6.73\times 10^{-5}$ & $6.41\times 10^{-5}$ &$9.44\times 10^{-6}$\\
  \hline
  \end{tabular}}
  \label{tab:problems}
\end{table}

The details of the parameters we set and some other important quantities can be found in Table \ref{tab:problems}. To show the predictive ability of our model, we pick $T_{predict}=20\pi$ for the pendulum, the Lotka--Volterra and the Kepler problems. For the  H\'enon--Heiles problem, we pick $T_{predict}=10$ because of its chaotic nature. We pick 15 as the sample size for the pendulum problem and 25 for other problems since we find that small $N_{train}$'s are sufficient to generate excellent results. More discussions about $N_{train}$ can be found in section \ref{sub:sample}. The epoch parameter represents the number of epochs needed for the training loss to converge. $step\_size$ indicates the period of learning rate decay, and $\gamma$ is the multiplicative factor of learning rate decay. These two parameters decay the learning rate of each parameter group by $\gamma$ every $step\_size$ epochs, which prevents the model from overshooting the local minimum. The dynamic learning rate can also make our model converge faster. $M$ indicates the number of terms of the Taylor polynomial introduced in the construction of the neural networks \eqref{eq:Tp_Taylor}. Through experimentation, we find that 8 terms can represent most functions well. Therefore, we pick $M=8$ for the pendulum and the Lotka-Volterra problems. For more complicated systems, like the Kepler and the  H\'enon--Heiles systems, we choose $M=20$ and $M=12$, respectively.

$N_h$, the dimension of hidden layers, is a parameter that needs to be carefully chosen.
We conduct the ablation test on the pendulum, the
Lotka--Volterra, the Kepler, and the H\'enon--Heiles problems to compare the validation loss using different $N_h$. Figure  \ref{fig:N_h} shows the results of the test. From figure  \ref{fig:N_h}(a), it can be seen that the validation loss after convergence for the pendulum problem drops significantly after increasing $N_h$ from $8$ to $16$ and then stays relatively similar with higher $N_h$. Therefore, we choose to use 16 as $N_h$ for the pendulum problem. Following the same logic, we choose $8$, $8$, and $16$ as $N_h$ for the Lotka--Volterra, the Kepler, and the H\'enon--Heiles problems. Notice that $N_h$ for the lower-dimensional problem, namely, the pendulum problem, is larger than $N_h$ for the higher dimensional problem, the Kepler problem. This is because, for the higher-dimensional problem, the degree of freedom is actually more limited. This is due to the prior knowledge that the forces between objects are the same.

\begin{figure}
    \centering
    \includegraphics[width=.96\linewidth]{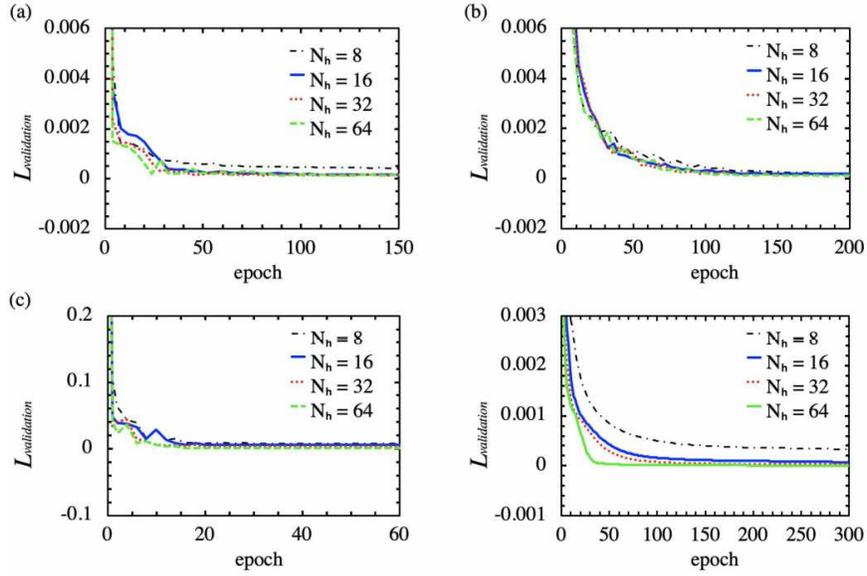}
    \caption{Comparisons of validation losses with different $N_h$ in the training process for (a) the pendulum, (b) the Lotka--Volterra, (c) the Kepler, and (d) the H\'enon--Heiles problems.}
    \label{fig:N_h}
\end{figure}

Another vital parameter that is not mentioned in Table \ref{tab:problems} is the integral time step $\Delta t$ in the sympletic integrator. Notice that the choice of $\Delta t$ largely depends on the time span $T_{train}$. Figure \ref{fig:different_timestep} compares the validation losses generated by various integral time steps $\Delta t$ based on fixed dataset time spans $T_{train}=0.01$, $0.1$ and $0.2$ respectively in the training process. For the concern of gradient vanishing or exploding, notice that when the number of iterations $n$ is big, which is when $\Delta t$ is small, we did not observe these issues, as shown in \ref{fig:different_timestep}(a), where the smallest $\Delta t$ is $10^{-4}$. Since we embed the structure of residual networks in our symplectic integrator, there should not be the problem of vanishing gradient. It is clear that the validation loss converges to a similar degree with various $\Delta t$ based on fixed $T_{train}=0.01$ and $T_{train} = 0.1$ in \ref{fig:different_timestep}(a) and (b), while it increases significantly as $\Delta t$ increases based on fixed $T_{train}=0.02$ in \ref{fig:different_timestep}(c). Thus, we need to be careful when choosing $n$, or $\Delta t$, for the dataset with larger time span $T_{train}$.

\begin{figure}
        \centering
        \includegraphics[width=.96\linewidth]{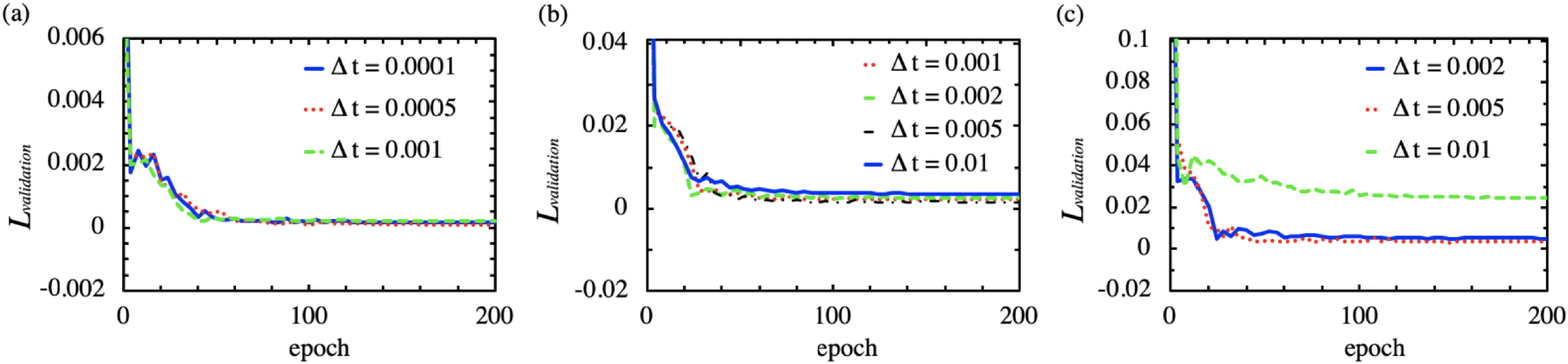}
        \caption{Comparisons of validation losses with different $dt$ in the training process. (a), (b), and (c) are trained based on different time spans $T_{train} = 0.01$, $0.1$, and $0.2$, respectively.}
     \label{fig:different_timestep}
\end{figure}

We record the training loss for all the problems at the epochs specified above. It is worth noticing that the training loss of our model is at $10^{-5}$ order of magnitude and below, which indicates our model's ability to fit the training data. As we can see from figure \ref{fig:Prediction}, the prediction results using Taylor-net match perfectly with the ground truth for all three systems, even though the $T_{train}=0.01$ is $2000\pi$ times shorter than the $T_{predict}=20\pi$ in figure \ref{fig:Prediction} (a) and (b), and $1000$ times shorter in figure \ref{fig:Prediction} (c). In particular, our model predicts the dynamics of the chaotic system, the H\'enon--Heiles system \eqref{eq:hh1} extremely well, which regular neural networks fail to do. The results indicate the compelling predictive ability of our model. This can be seen more clearly in \ref{sub:compare} when we compare Taylor-net with other methods.

\begin{figure}
  \centering
  \psfrag{m}{\scriptsize Ground Truth}
  \psfrag{n}{\scriptsize Prediction}
  \psfrag{x}[c][c]{\footnotesize $\bm  {q}$}
  \psfrag{y}[c][c]{\footnotesize $\bm  p$}
  \psfrag{e}[c][c]{\footnotesize $q_1$}
  \psfrag{f}[c][c]{\footnotesize $q_2$}
  \psfrag{a}[c][c]{\footnotesize (a)}
  \psfrag{b}[c][c]{\footnotesize (b)}
  \psfrag{c}[c][c]{\footnotesize (c)}
  \includegraphics[width=.32\linewidth]{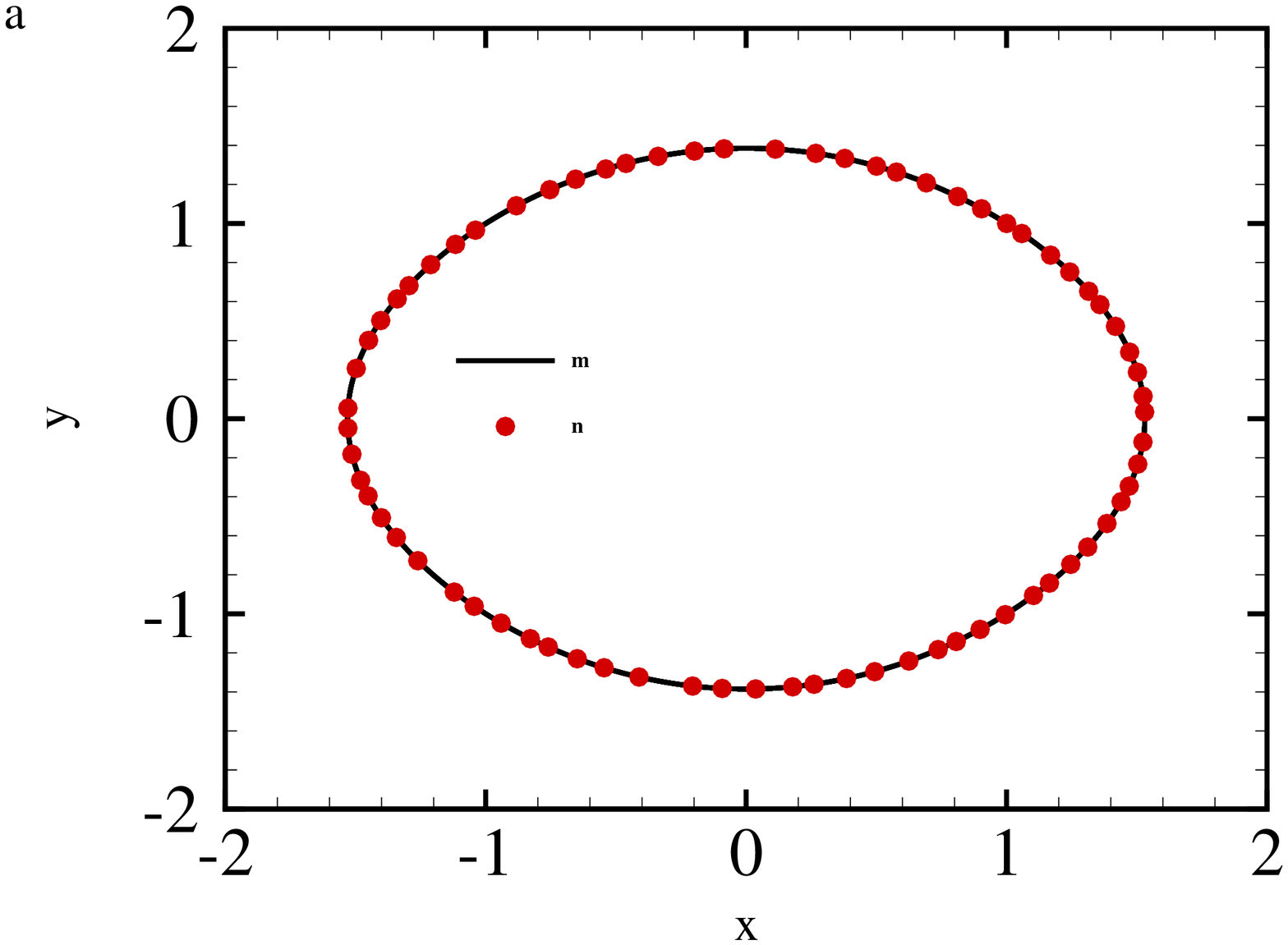}\hfill
  \includegraphics[width=.32\linewidth]{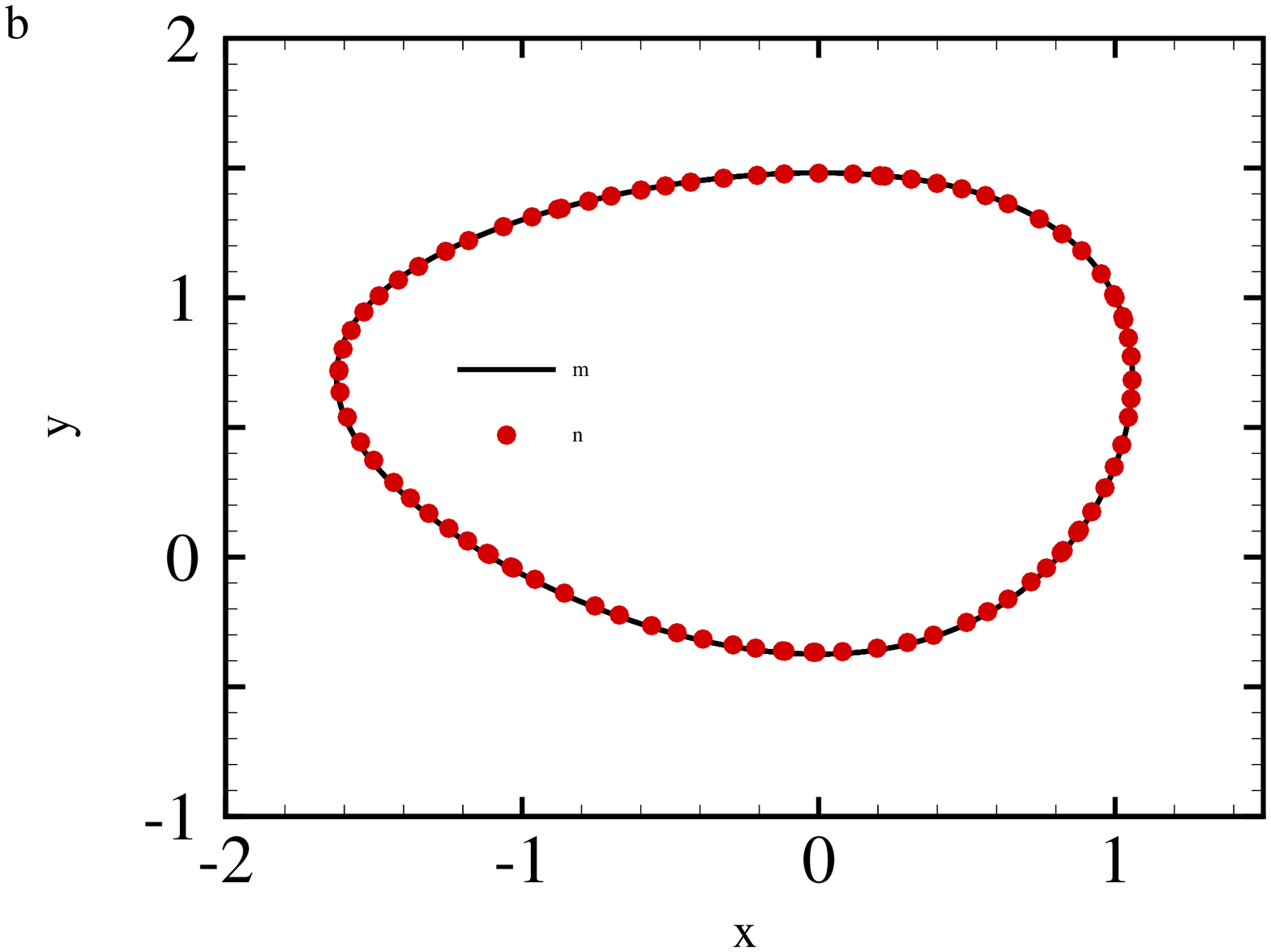}\hfill
  \includegraphics[width=.32\linewidth]{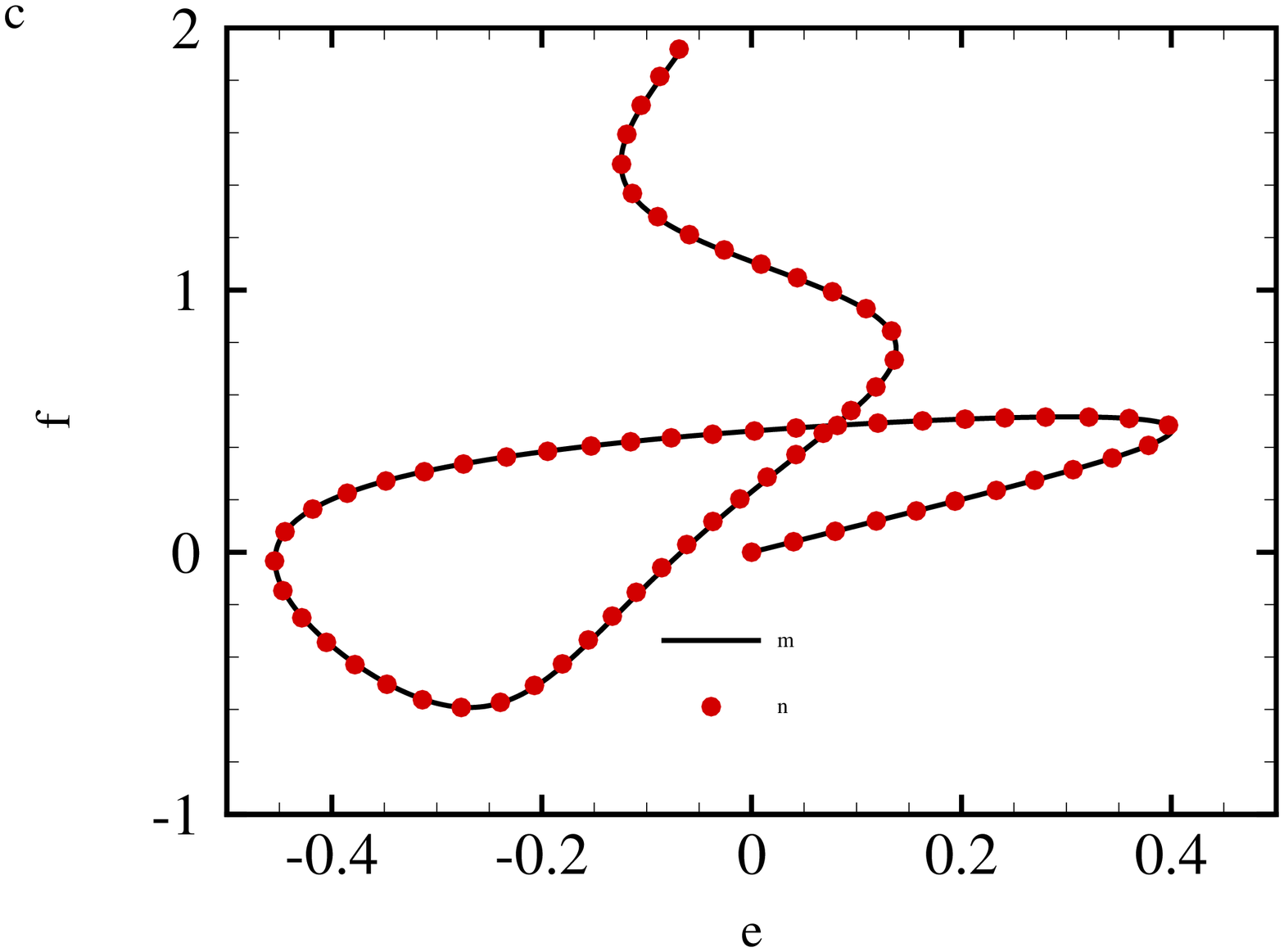}\\[0.5mm]
  \caption{Prediction result using Taylor-net for (a) the pendulum, (b) the Lotka--Volterra, and (c) the H\'enon--Heiles problems. For better visualization, we set the initial points as (a) $(\bm{q}_0,\bm{p}_0)=(1,1)$, (b) $(\bm{q}_0,\bm{p}_0)=(1,1)$, and (c) $(\bm{q}_0,\bm{p}_0)=([0, 0],[0.5, 0.5])$.  The prediction results using Taylor-net match perfectly with the ground truth for all three systems, even though the $T_{train}$ is $2000\pi$ times shorter than the $T_{predict}$ in (a) and (b), and $1000$ times shorter in (c). $T_{train}=0.01$ and $T_{predict}=20\pi$ in (a) and (b), and $T_{train}=0.01$ and $T_{predict}=10$ in (c).}.
  \label{fig:Prediction}
\end{figure}

\subsection{Taylor series vs. ReLU}

In order to evaluate the performance of using Taylor series as the underlying structure of Taylor-net to ensure nonlinearity, we also implement the most commonly used activation function, ReLU and compare the training loss with our current model. We construct the neural networks as \eqref{eq:Tp_Taylor} with parameters specified in Table \ref{tab:problems}, except we use $f_i(x)=\max(0,x)$ instead. The experimental results show that the neural networks perform better with Taylor series than with ReLU in the pendulum, the Lotka--Volterra, and the Kepler problems. We can observe from figure \ref{fig:vsReLU} that in all three problems the loss of using ReLU is larger than the loss of using Taylor series after the loss converges. In the pendulum problem, the mean of loss after convergence from 100 epochs to 300 epochs using the Taylor series is $8.878\times 10^{-5}$, while that of using ReLU is $8.348\times 10^{-4}$, which is 10 times larger than the mean of loss using Taylor series. The difference in the Lotka--Volterra problem is even more obvious. The mean of loss from 100 epochs to 300 epochs using the Taylor series is $7.832\times 10^{-5}$, while that of using ReLU is $4.782\times 10^{-3}$. In the Kepler problem, the mean of loss from 40 epochs to 100 epochs using the Taylor series is $2.524\times 10^{-4}$, while that of using ReLU is $8.408\times 10^{-4}$. In all three problems, the Taylor series performs undoubtedly better than ReLU. Thus, the results clearly show that using the Taylor series gives a better approximation of the dynamics of the system. The strong representational ability of the Taylor series is an important factor that increases the accuracy of the prediction.

\begin{figure}
  \centering
  \psfrag{m}{\scriptsize Taylor Series}
  \psfrag{n}{\scriptsize ReLU}
  \psfrag{x}[c][c]{\footnotesize epoch}
  \psfrag{y}[c][c]{\footnotesize $L_{train}$}
  \psfrag{a}[c][c]{\footnotesize (a)}
  \psfrag{b}[c][c]{\footnotesize (b)}
  \psfrag{c}[c][c]{\footnotesize (c)}
  \includegraphics[width=.32\linewidth]{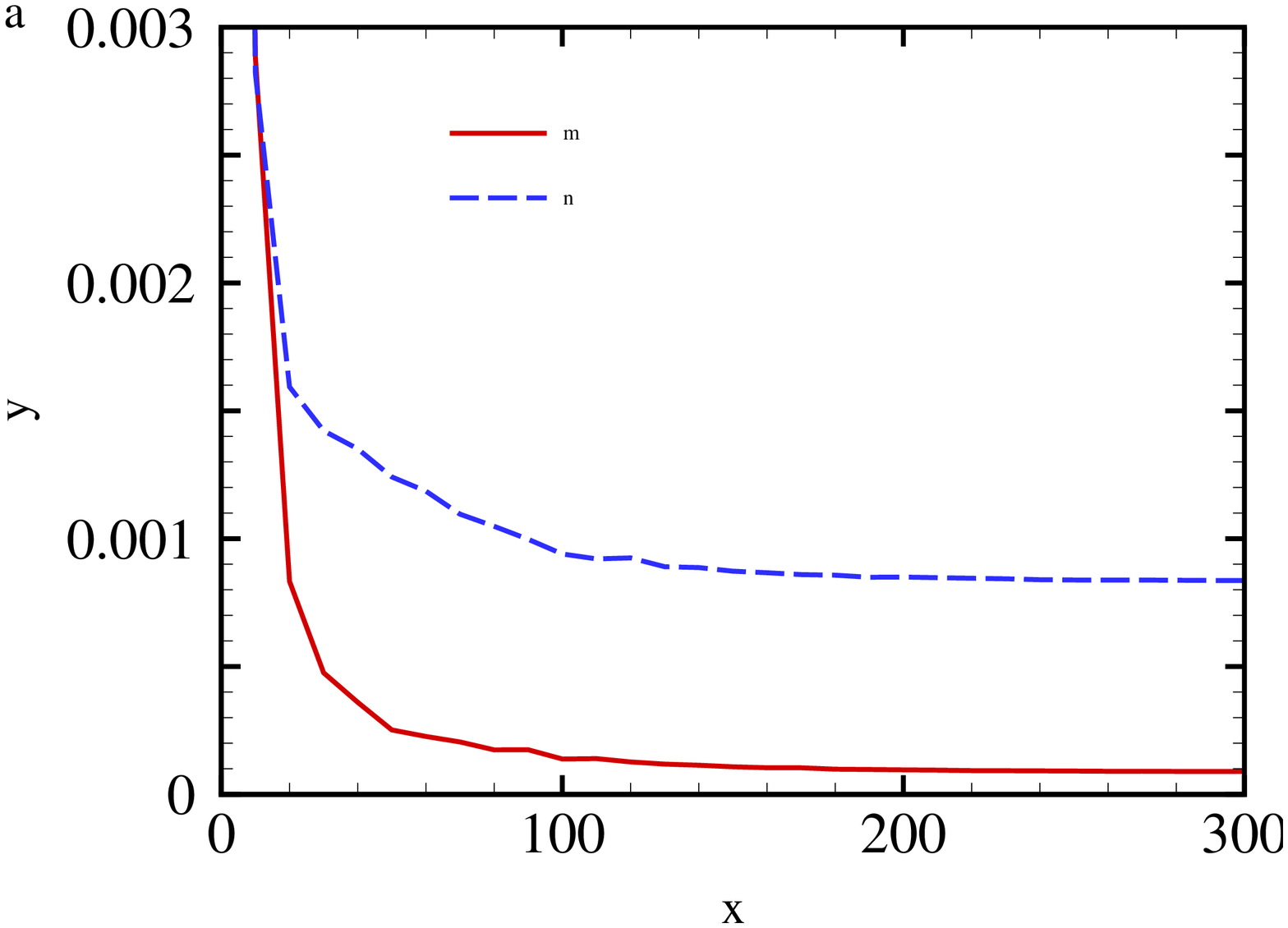}\hfill
  \includegraphics[width=.32\linewidth]{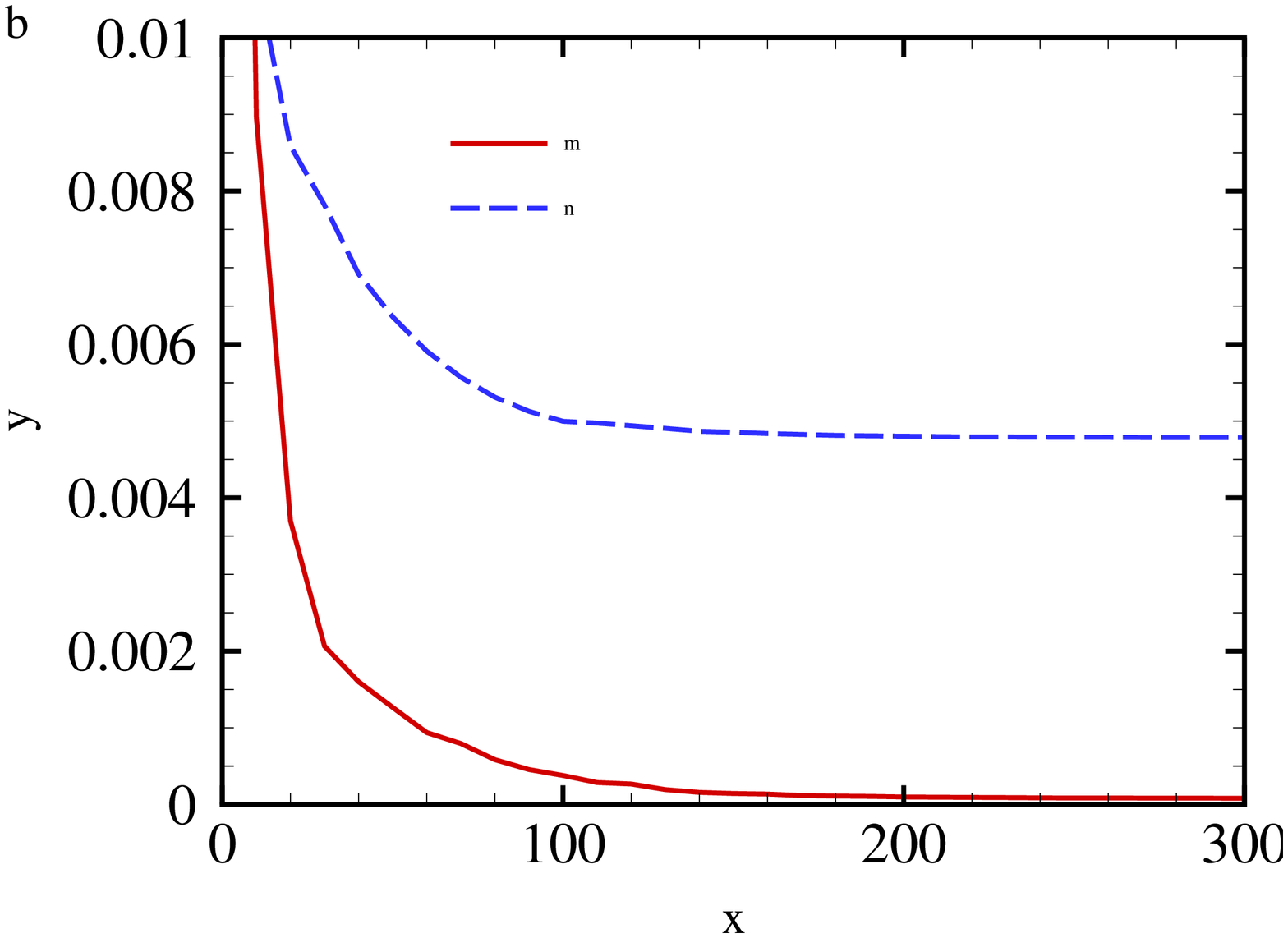}\hfill
  \includegraphics[width=.32\linewidth]{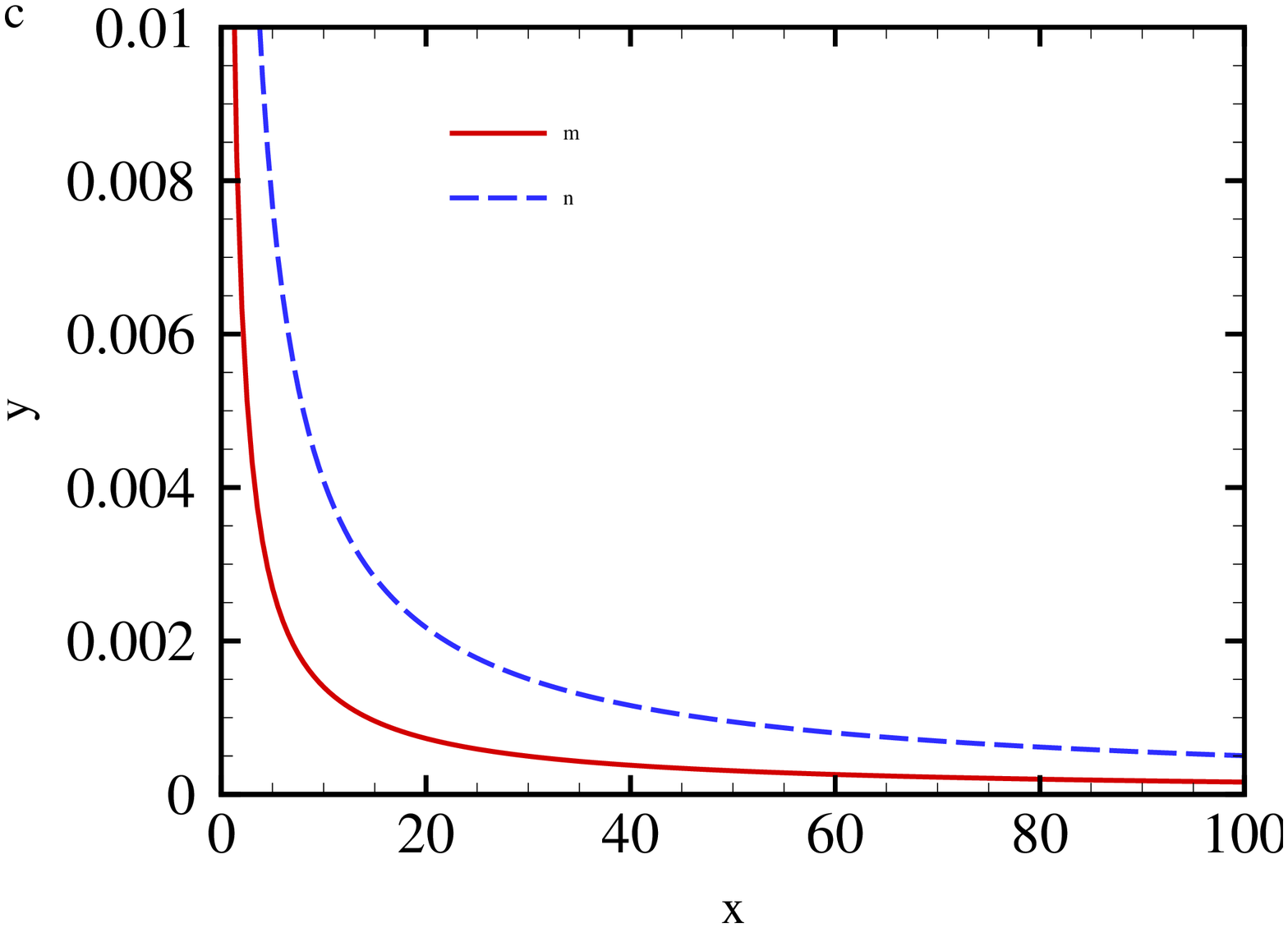}\\[0.5mm]
  \caption{Mean of $L_{train}$ using Taylor series vs. using ReLU for (a) the pendulum, (b) the Lotka--Volterra, and (c) the Kepler problems. We train each model until $L_{train}$ converges and average $L_{train}$ for (a) every 10 epochs for the pendulum problem, (b) every 10 epochs for the Lotka--Volterra problem, and (c) every 5 epochs for the Kepler problem.}
  \label{fig:vsReLU}
\end{figure}

\subsection{Predictive ability and robustness}\label{sub:compare}

Now, to assess how well our method can predict the future flow, we compare the predictive ability of Taylor-net with ODE-net and HNN. We apply all three methods on the pendulum problem, and let $T_{train}=0.01$ and $T_{predict}=20\pi$. We evaluate the performance of the models by calculating the average prediction error at each predicted points, defined by
\begin{equation}
  \epsilon_p^{(n_t)}=\frac{1}{N_{test}}\sum_{s=1}^{N_{test}}\|\bm{\hat{p}}^{(s, n_t)}_n-\bm{p}_n^{(s, n_t)}\|_1+\|\bm{\hat{q}}_n^{(s,  n_t)}-\bm{q}_n^{(s, n_t)}\|_1,
\end{equation}
and the average $\epsilon_p^{(n_t)}$ over $T_{predict}$ is
\begin{equation}
  \epsilon_p=\frac{1}{N_T}\sum_{n_t=1}^{N_T}\epsilon_p^{(n_t)},
\end{equation}
where $N_{test}$ represents the testing sample size specified in section \ref{subsec:data} and $N_T=T_{predict}/\Delta t$ with $\Delta t = 0.01$. After experimentation, we find that Taylor-net has stronger predictive ability than the other two methods. The first row of Table \ref{tab:compa_err} shows the average prediction error of 100 testing samples using the three methods over $T_{predict}$ when no noise is added. The prediction error of HNN is almost double that of Taylor-net, while the prediction error of ODE-net is about 7 times that of Taylor-net. To analyze the difference more quantitatively, we made several plots to help us better compare the prediction results. Figure \ref{fig:Error} shows the plots of prediction error $\epsilon_p^{(n_t)}$ against $t=n_t \Delta t$ over $T_{predict}$ for all three methods. In figure \ref{fig:prediction_q}, we plot the prediction of position $q$ against time period for all three methods as well as the ground truth in order to see how well the prediction results match the ground truth. From figure \ref{fig:prediction_q} (a), we can already see that the prediction result of ODE-net gradually deviates from the ground truth as time progresses, while the prediction of Taylor-net and HNN stays mostly consistent with the ground truth, with the former being slightly closer to the ground truth. The difference between Taylor-net and HNN can be seen more clearly in figure \ref{fig:Error} (a). Observe that the prediction error of Taylor-net is obviously smaller than that of the other two methods, and the difference becomes more and more apparent as time increases. The prediction error of ODE-net is larger than HNN and Taylor-net at the beginning of $T_{predict}$ and increases at a much faster rate than the other two methods. Although the prediction error of HNN has no obvious difference from that of Taylor-net at the beginning, it gradually diverges from the prediction error of Taylor-net.
\begin{table}
  \caption{Comparison of $\epsilon_p$ for the pendulum problem without noise, with noise $\sigma_1, \sigma_2 \sim \mathcal{N}(0,0.1)$, and with noise $\sigma_1, \sigma_2 \sim \mathcal{N}(0,0.5)$.}
  \centering
  \setlength{\tabcolsep}{1mm}{
  \begin{tabular}{lccc}
  \hline
  Methods & Taylor-net & HNN & ODE-net\\
  \hline
  $\epsilon_p$, without noise & 0.213 &0.377&1.416\\
  $\epsilon_p$, with noise $\sigma_1, \sigma_2 \sim \mathcal{N}(0,0.1)$ & 1.667 &2.433&3.301\\
  $\epsilon_p$, with noise $\sigma_1, \sigma_2 \sim \mathcal{N}(0,0.5)$&1.293  &2.416 & 27.114\\
  \hline
  \end{tabular}}
  \label{tab:compa_err}
\end{table}
\begin{figure}
  \centering
  \psfrag{m}{\scriptsize Taylor-net}
    \psfrag{n}{\scriptsize HNN}
    \psfrag{l}{\scriptsize ODE-net}
    \psfrag{x}[c][c]{\footnotesize $t$}
    \psfrag{y}[c][c]{\footnotesize $\epsilon_p^{(n_t)}$}
    \psfrag{a}[c][c]{\footnotesize (a)}
   \psfrag{b}[c][c]{\footnotesize (b)}
   \psfrag{c}[c][c]{\footnotesize (c)}
  \includegraphics[width=.32\linewidth]{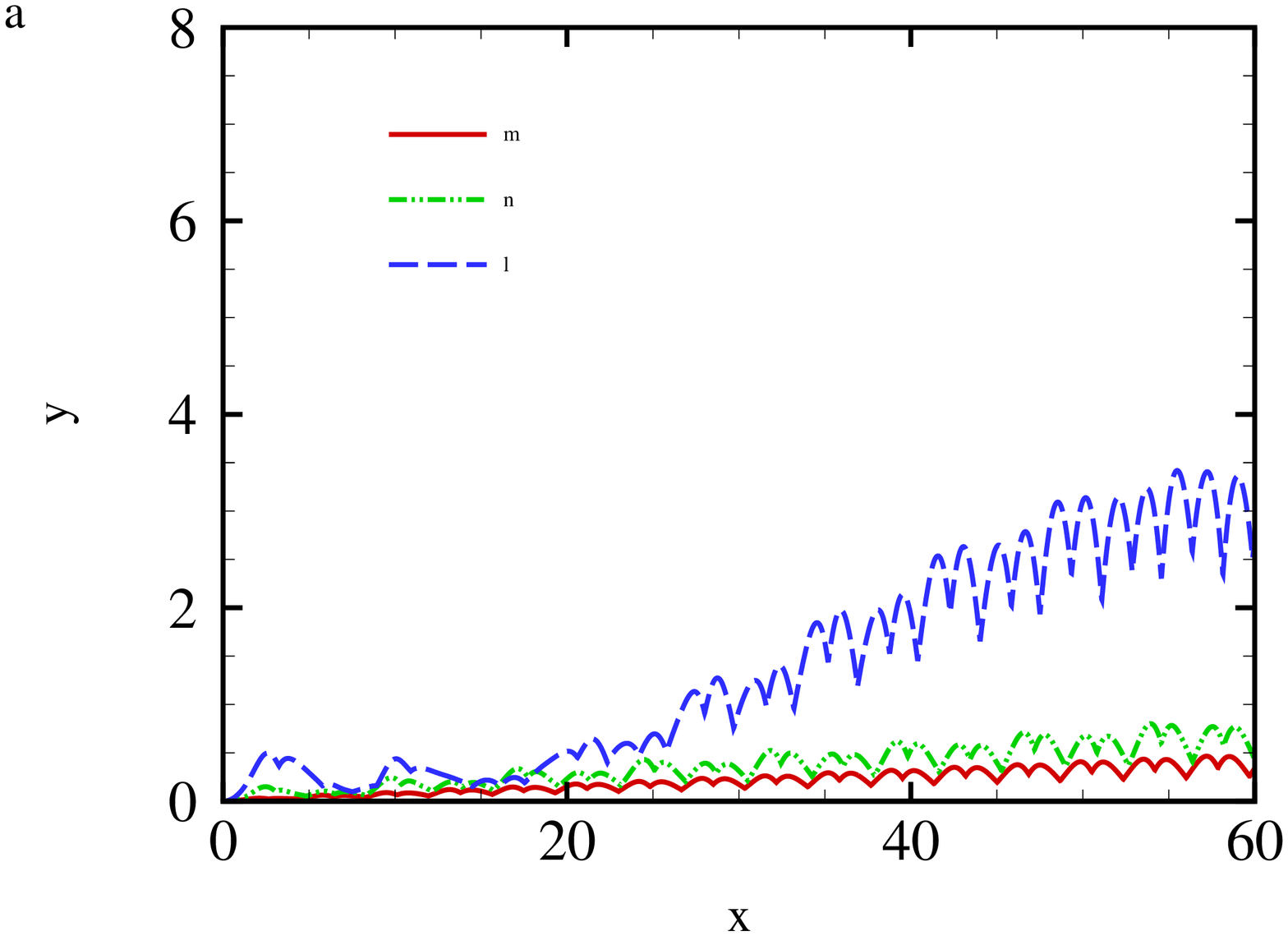}\hfill
  \includegraphics[width=.32\linewidth]{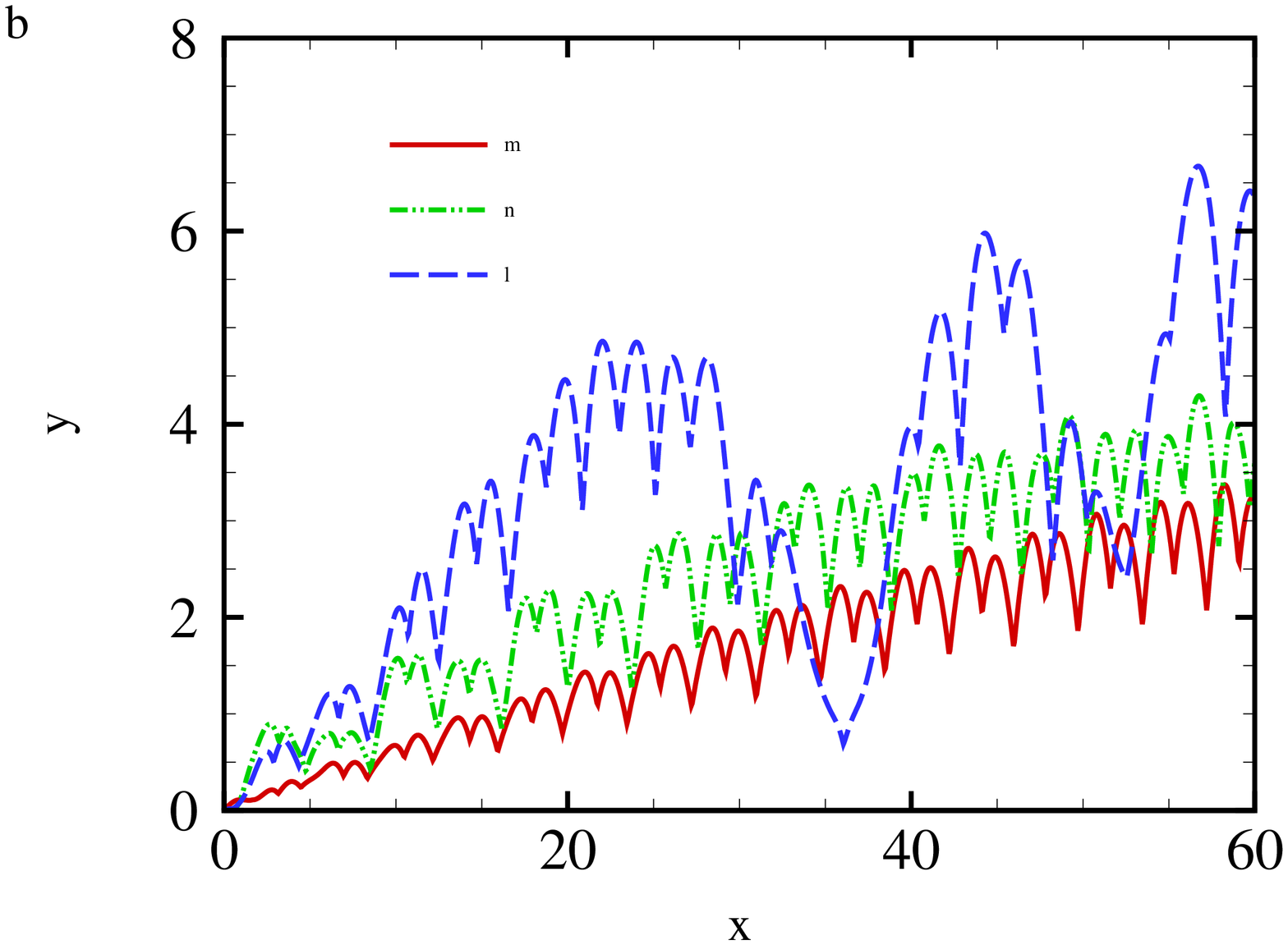}\hfill
  \includegraphics[width=.32\linewidth]{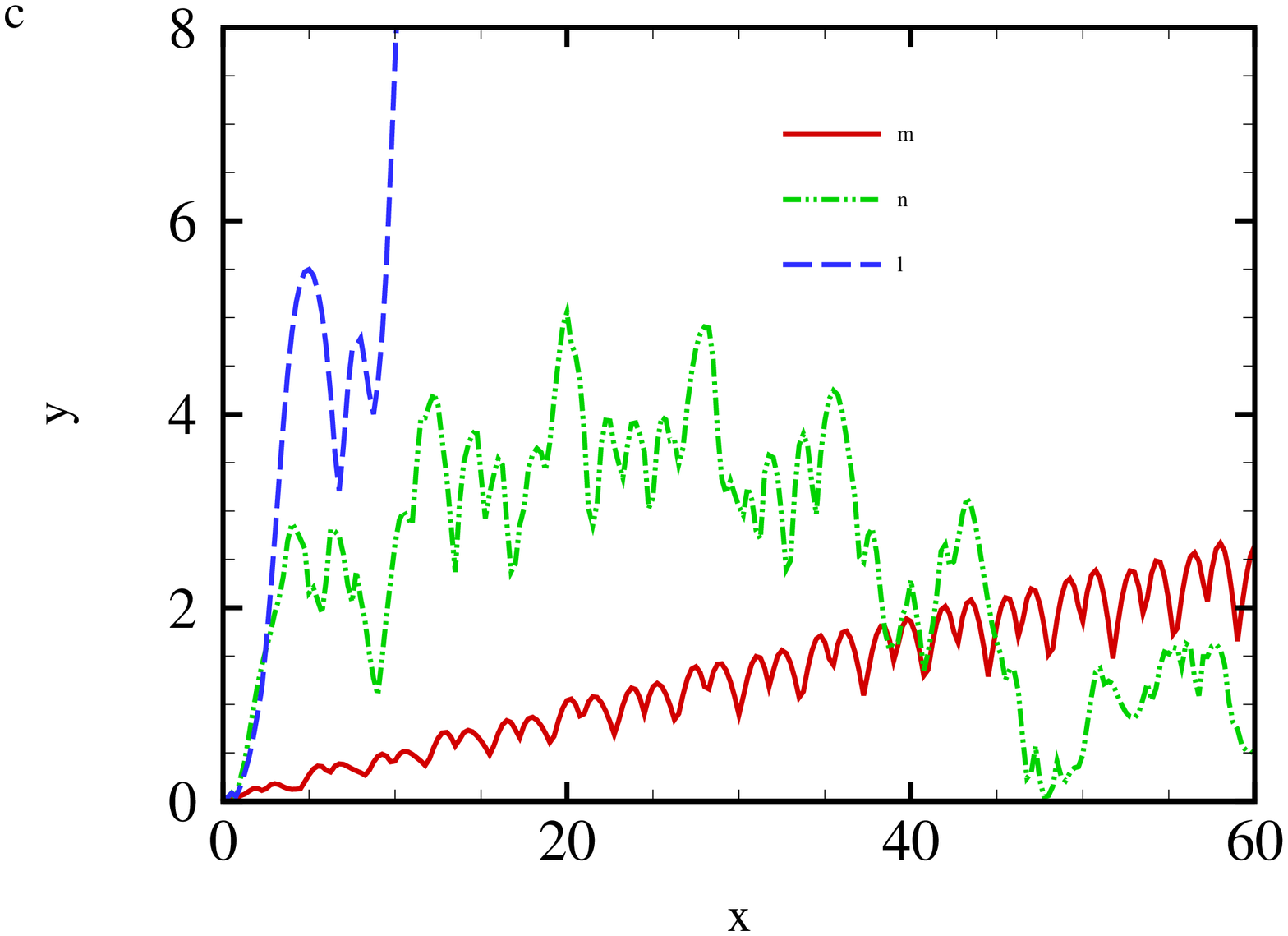}
  \\[0.5mm]
  \caption{Prediction error $\epsilon_p^{(n_t)}$ at different $t$ from $t=0$ to $t=20\pi$ for the pendulum problem (a) without noise, (b) with noise $\sigma_1, \sigma_2 \sim \mathcal{N}(0,0.1)$, and (c) with noise $\sigma_1, \sigma_2 \sim \mathcal{N}(0,0.5)$. In the figure, $t=n_t \Delta t$, where $\Delta t=0.01$. $\epsilon_p^{(n_t)}$ is the prediction error at the $n_t^{\textrm{th}}$ predicted point among the total $N_T=T_{predict}/\Delta t$ predicted points. We use $T_{train}=0.01$, $T_{train}=0.5$ and $T_{train}=1$ to train the model in (a), (b), and (c), respectively.}
  \label{fig:Error}
\end{figure}

\begin{figure}
  \centering
   \psfrag{m}{\scriptsize Ground Truth}
    \psfrag{n}{\scriptsize Taylor-net}
    \psfrag{l}{\scriptsize HNN}
    \psfrag{o}{\scriptsize ODE-net}
    \psfrag{a}[c][c]{\footnotesize (a)}
    \psfrag{b}[c][c]{\footnotesize (b)}
    \psfrag{c}[c][c]{\footnotesize (c)}
    \psfrag{x}[c][c]{\footnotesize $t$}
    \psfrag{y}[c][c]{\footnotesize $\bm q$}
  \includegraphics[width=0.92\linewidth]{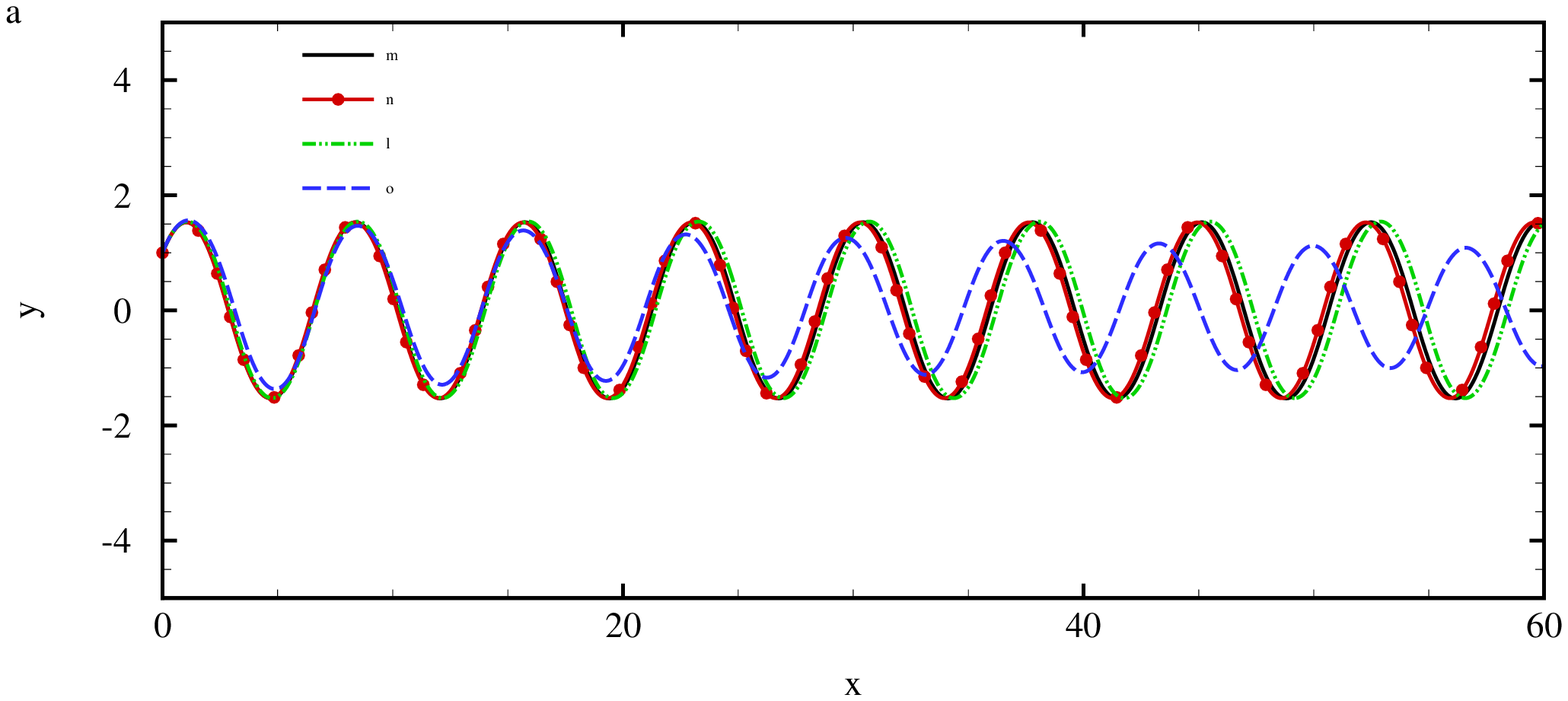}\\
  \includegraphics[width=0.92\linewidth]{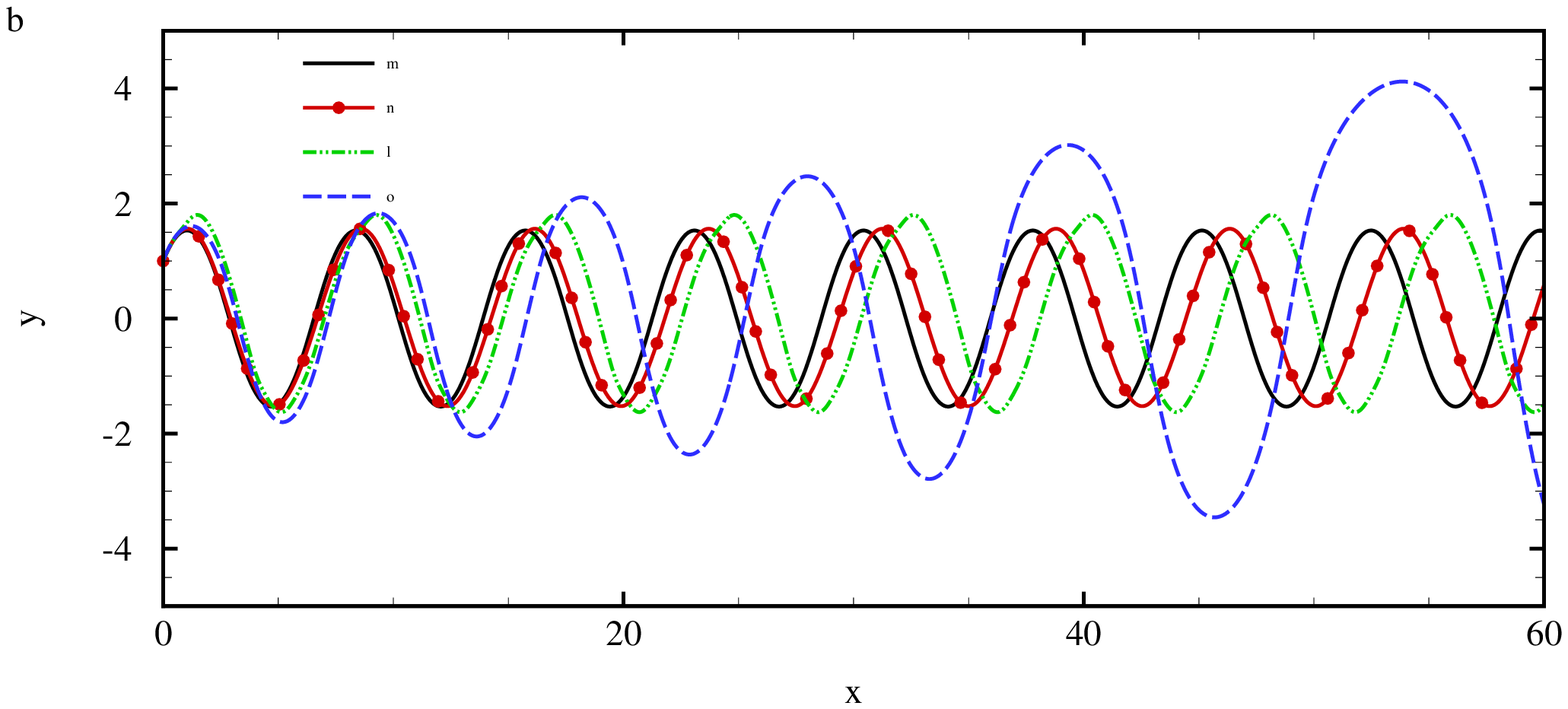}\\
  \includegraphics[width=0.92\linewidth]{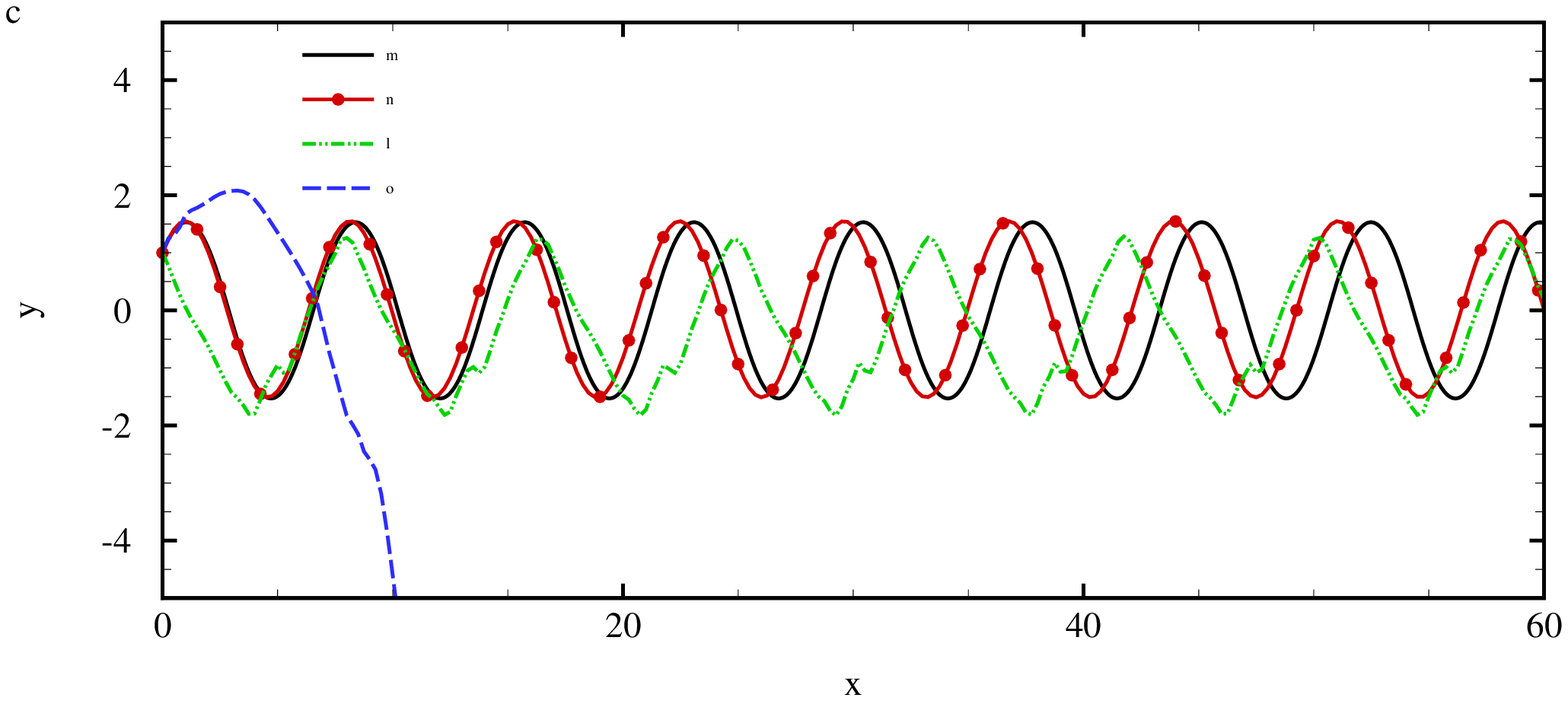}\\
  \caption{Prediction results of position $\bm q$ from $t=0$ to $t=20\pi$ for the pendulum problem using Taylor-net, HNN, and ODE-net  (a) without noise, (b) with noise $\sigma_1, \sigma_2 \sim \mathcal{N}(0,0.1)$, and (c) with noise $\sigma_1, \sigma_2 \sim \mathcal{N}(0,0.5)$. For all the models, we set the initial point as $(\bm{q}_0,\bm{p}_0)=(1,1)$. We use $T_{train}=0.01$, $T_{train}=0.5$ and $T_{train}=1$ to train the model in (a), (b), and (c), respectively. All the methods are trained until the $L_{validation}$ converges. }
  \label{fig:prediction_q}
\end{figure}

Additionally, in figure \ref{fig:Hamiltonian}, we plot the numerically solved ground truth, Taylor-net, HNN, and ODE-net calculated Hamiltonian for the pendulum problem. Figure \ref{fig:Hamiltonian} shows that the Taylor-net preserves the Hamiltonian relatively successfully, while ODE-net diverges away from the ground truth quickly. Although the predicting result of HNN does not seem to drift away from the ground truth, the divergent amplitude of HNN is greater than that of Taylor-net. Note that our model strictly preserves the symplectic structure, which is a geometric structure that cannot be quantitatively calculated and plotted. Since the symplectic structure is preserved, the Hamiltonian predicted by our model is much closer to the ground truth.

\begin{figure}
        \centering
        \includegraphics[width=.96\linewidth]{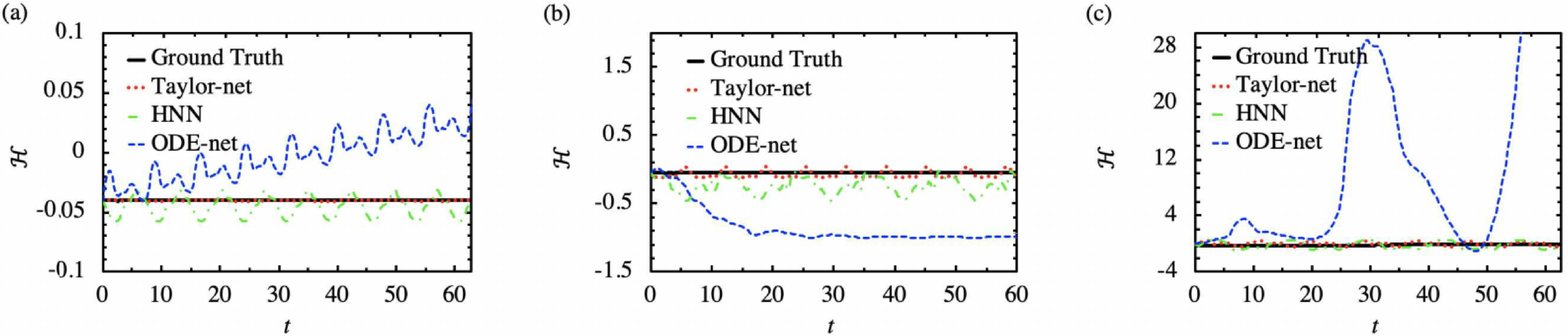}
    	\caption{Prediction results of Hamiltonian $\mathcal{H}$ from $t=0$ to $t=12\pi$ for the pendulum problem (a) without noise, (b) with noise $\sigma_1, \sigma_2 \sim \mathcal{N}(0,0.1)$, and (c) with noise $\sigma_1, \sigma_2 \sim \mathcal{N}(0,0.5)$.}
    	\label{fig:Hamiltonian}
\end{figure}

In real systems, it is almost impossible to collect data without noise. Therefore, with noisy data, the robustness of neural networks is particular important. Instead of using $(\bm{q}_{n},\bm{p}_{n})$ to train the model, we add some random noise to the true value so that it becomes $(\bm{q}_{n}+\sigma_1,\bm{p}_{n}+\sigma_2)$.
We test three models on two cases with small and large noises. We add noise $\sigma_1, \sigma_2 \sim \mathcal{N}(0,0.1)$ in the case of small noise and $\sigma_1, \sigma_2 \sim \mathcal{N}(0,0.5)$ in the case of large noise.  We use $T_{train}=0.5$ and $T_{train}=1$ to train the model in the cases of small and large noises respectively. In both cases, we use 50 samples and make prediction over $T_{predict}=20\pi$.

Figure \ref{fig:Prediction_result} shows the predicted $\bm p$ versus $\bm q$ using different methods. From figure \ref{fig:Prediction_result} (a), we find that Taylor-net discovers the unknown trajectory successfully, while ODE-net diverges from the true value quickly. Although the predicting result of HNN does not seem to drift away from the true dynamics, it does not fit the true trajectory as well as the prediction made by Taylor-net. The difference becomes clearer as we increase the noise. From figure \ref{fig:Prediction_result} (b), we observe that Taylor-net still makes predictions that are almost consistent with the true trajectories, while ODE-net completely fails to do so. Moreover, the prediction made by HNN is much worse than in the case of small noise, while the performance of Taylor-net remains as good as the previous case.

\begin{figure}
  \centering
  \psfrag{m}{\scriptsize Ground Truth}
  \psfrag{n}{\scriptsize Taylor-net}
  \psfrag{l}{\scriptsize HNN}
  \psfrag{o}{\scriptsize ODE-net}
  \psfrag{a}[c][c]{\footnotesize (a)}
  \psfrag{b}[c][c]{\footnotesize (b)}
  \psfrag{x}[c][c]{\footnotesize $\bm q$}
  \psfrag{y}[c][c]{\footnotesize $\bm p$}
  \includegraphics[width=.48\linewidth]{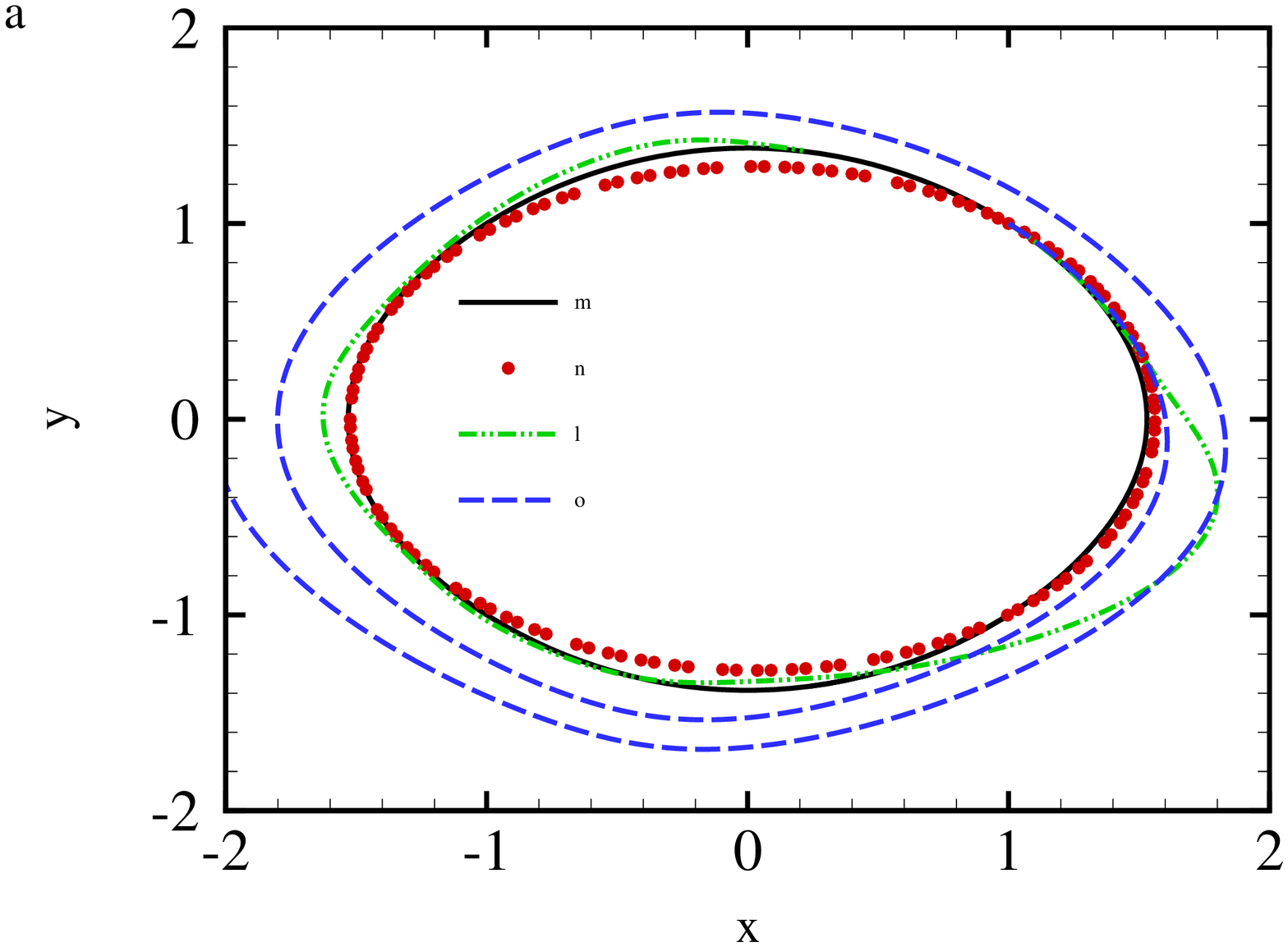}\hfill
  \includegraphics[width=.48\linewidth]{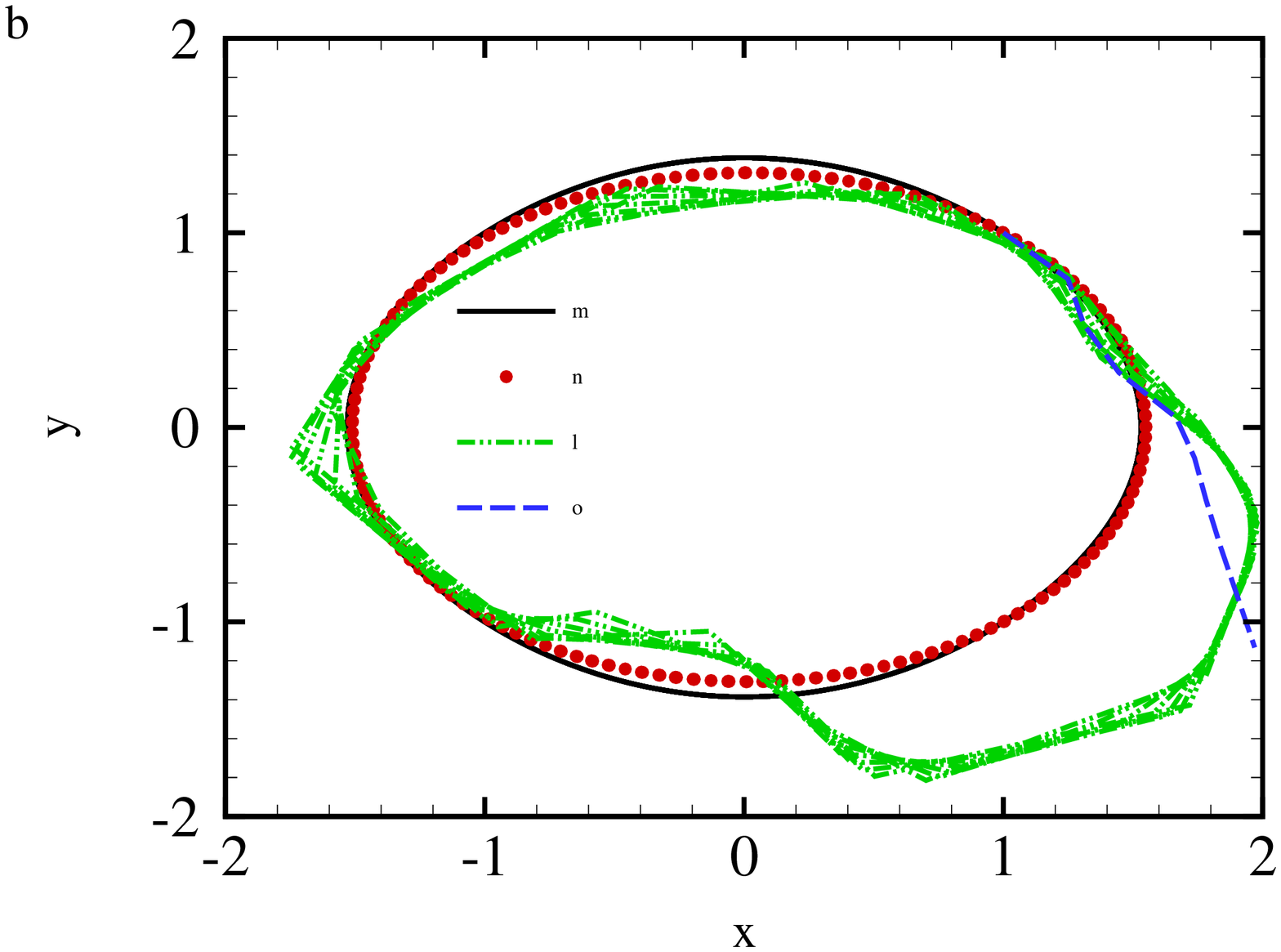}\\[0.5mm]
  \caption{Prediction results of position $\bm q$ and momentum $\bm p$ from $t=0$ to $t=20\pi$. (a) with noise $\sigma_1, \sigma_2 \sim \mathcal{N}(0,0.1)$ and (b) with noise $\sigma_1, \sigma_2 \sim \mathcal{N}(0,0.5)$ in training process. We use $T_{train}=0.5$ and $T_{train}=1$ to train the model in (a) and (b) respectively. All the methods are trained until the  $L_{validation}$ converges. In (a), we only plot the result of ODE-net until $t=4\pi$ because the result beyond that will further diverge from the ground truth and cannot be fit into the graph. For the same reason, we only plot the result of ODE-net until $t=\pi$ in (b).}
  \label{fig:Prediction_result}
\end{figure}

This can be more clearly seen from figure \ref{fig:Error} (b) and (c). We can see that $\epsilon_p^{(n_t)}$ of Taylor-net is consistent in both cases of small and large noises, while $\epsilon_p^{(n_t)}$ of HNN and ODE-net increase significantly and exhibit more fluctuation. It is worth noticing that in figure \ref{fig:Error} (c), $\epsilon_p^{(n_t)}$ of HNN becomes smaller towards the end of $T_{predict}$. However, it is not because the performance of HNN becomes better, but rather due to the fact that the predicted flow of HNN is off by one period of motion, which can be seen from figure \ref{fig:prediction_q} (c). The second and third rows of Table \ref{tab:compa_err} also give an overview on how the prediction error $\epsilon_p$ over $T_{predict}$ of the three methods differ. From figure \ref{fig:prediction_q} (b) and (c), we can clearly observe that the amplitude of predicted $\bm q$ using ODE-net increases as $t$ increases, and the amplitude of predicted $\bm q$ using HNN is slightly larger or smaller than that of the ground truth from the beginning. In contrast, due to the intrinsic symplectic structure of Taylor-net, the amplitude of predicted $\bm q$ using Taylor-net is inconsistent with the ground truth, without changing in time. Additionally, it is obvious that the predicted $\bm q$ using Taylor-net has the smallest phase shift among the three methods.

\subsection{Training sample size and convergence rate}\label{sub:sample}
Besides the strong predictive ability and robustness, we also want to highlight the significantly small $N_{train}$ and the fast convergence rate of our approach. In a complex physical system, the cost of acquiring data is high. Our model can learn from the dataset that contains less than 15 samples and still generate validation loss $L_{validation}$ that is below $10^{-4}$. In figure \ref{fig:sample_epoch} (a), we plot the $L_{validation}$ as a function of sample size using Taylor-net, HNN, and ODE-net. To make a fair comparison, we average the values of $L_{validation}$ over 50 trials. We can observe that the $L_{validation}$ for Taylor-net at 1 sample is around 5 times smaller than the $L_{validation}$ for ODE-net and the $L_{validation}$ for HNN. Although there are some fluctuations in $L_{validation}$ due to small $N_{train}$, the $L_{validation}$ for Taylor-net converges at around 10 samples, while the $L_{validation}$ for HNN is still decreasing. Although the $L_{validation}$ for ODE-net also converges around 10 samples, the value of its $L_{validation}$ is 10 times larger than that the $L_{validation}$ for Taylor-net.

Because of the intrinsically structure-preserving nature of our model, our model can well predict the dynamics of the underlying system even when it is trained for only a few epochs. In figure \ref{fig:sample_epoch} (b), we plot the prediction results from $t=0$ to $t=20\pi$ using Taylor-net, HNN, and ODE-net after only 1 epoch of training. The prediction results made by HNN and ODE-net completely fail to match the true flow, while Taylor-net predicts the truth to a level that can never be achieved using HNN and ODE-net at such a small number of epochs.
\begin{figure}
  \centering
  \psfrag{m}{\scriptsize Ground Truth}
  \psfrag{n}{\scriptsize Taylor-net}
  \psfrag{l}{\scriptsize HNN}
  \psfrag{o}{\scriptsize ODE-net}
  \psfrag{a}[c][c]{\footnotesize (a)}
  \psfrag{b}[c][c]{\footnotesize (b)}
  \psfrag{x}[c][c]{\footnotesize $\bm q$}
  \psfrag{y}[c][c]{\footnotesize $\bm p$}
  \psfrag{z}[c][c]{\footnotesize $L_{validation}$}
  \psfrag{w}[c][c]{\footnotesize $N_{train}$}
  \includegraphics[width=.48\linewidth]{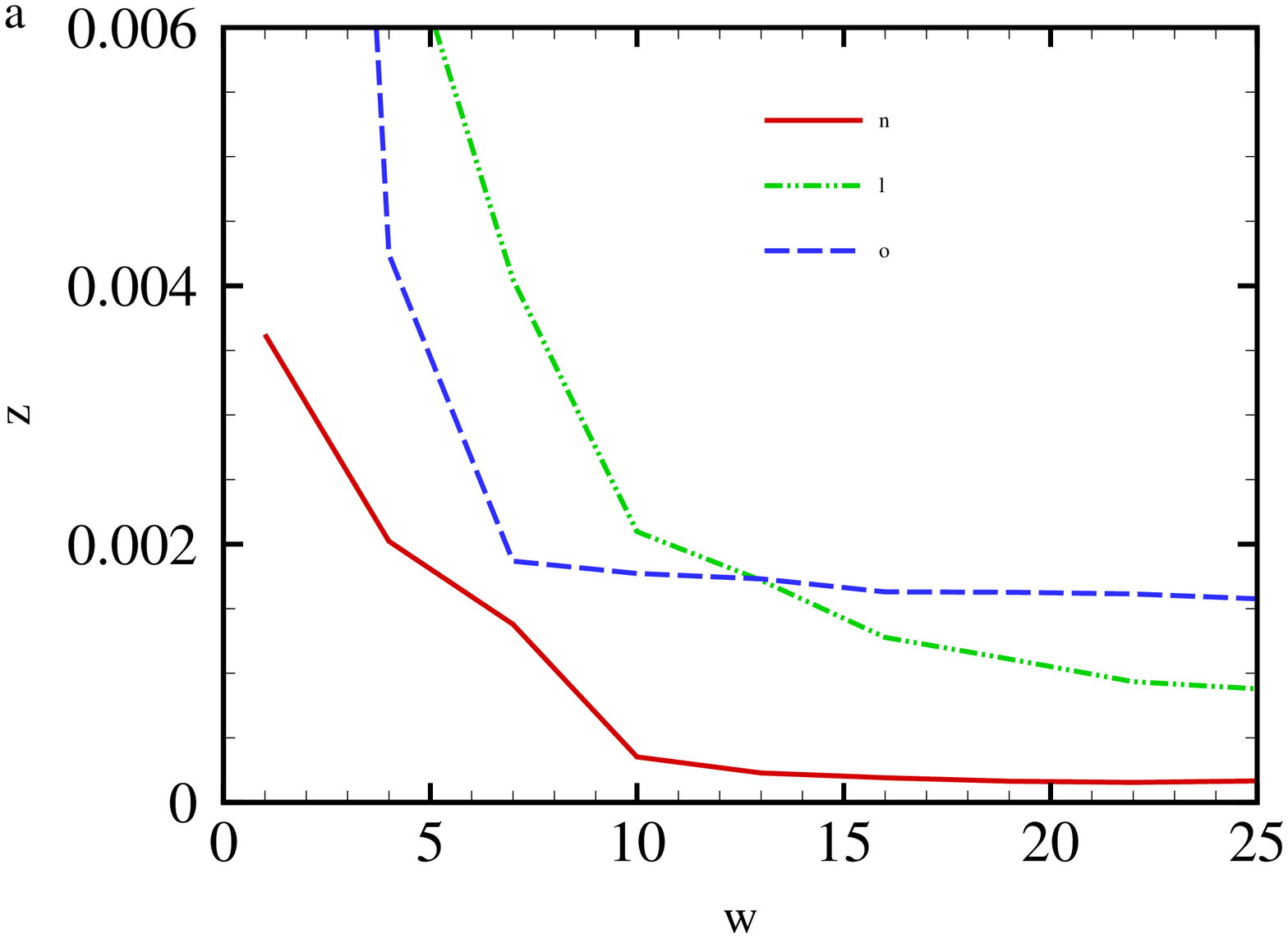}\hfill
  \includegraphics[width=.48\linewidth]{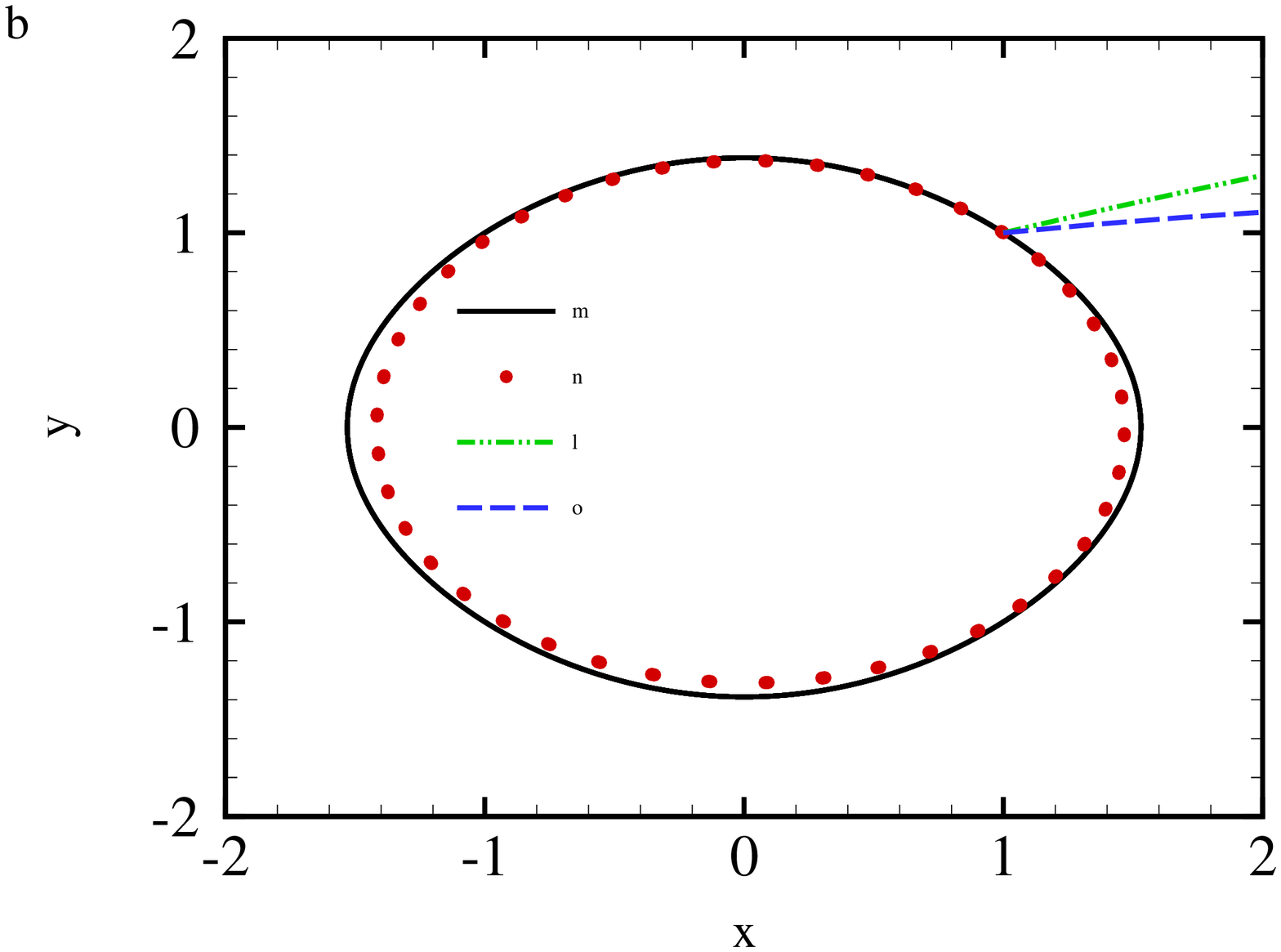}\\[0.5mm]
  \caption{(a) At 100 epochs, $L_{validation}$ as a function of sample size ranging from $N_{train}=1$ to $N_{train}=25$. The $L_{validation}$ is averaged over 50 trials. (b) Prediction results of position $\bm q$ and momentum $\bm p$ from $t=0$ to $t=2\pi$ from $t=0$ to $t=20\pi$ using trained models after 1 epoch.}
  \label{fig:sample_epoch}
\end{figure}

We summarize the main traits and performance of the three methodologies in Table \ref{tab:comparison}. We already emphasized enough that our model utilizes physics prior through constructing neural networks that intrinsically preserve the symplectic structure. Due to our model's structure-preserving ability, it can make accurate predictions with a very small training dataset that does not require any intermediate data. We also want to mention that HNN and ODE-net both require the analytical solutions of the temporal derivatives to train their models, which are often not obtainable from real systems. Moreover, besides the qualitative differences, we also compare the three methods quantitatively.
In the pendulum problem, we fix the sample size to be 15 and find Taylor-net only needs 100 epochs for $L_{train}$ to converge, while HNN and ODE-net need 1000 epochs and 7000 epochs respectively. We also test how many samples Taylor-net, HNN, and ODE-net need for $L_{validation}$ to decrease to $10^{-4}$. Notice that we train Taylor-net, HNN, and ODE-net until convergence, which is for 100, 1000, and 7000 epochs respectively. Taylor-net only needs 15 samples and 100 epochs of training to achieve $L_{validation}\sim 10^{-4}$, while HNN needs 50 samples and 1000 epochs and ODE-net needs 50 samples and 7000 epochs. If we train HNN and ODE-net for 100 epochs in the same manner as Taylor-net, their $L_{validation}$ will never reach $10^{-4}$.

\begin{table}
  \caption{Comparison between Taylor-net, HNN, ODE-net. $\checkmark$ represents the method preserves such property.  }
  \centering
  \begin{threeparttable}
  \setlength{\tabcolsep}{1mm}{
  \begin{tabular}{lccc}
  \hline
  Methods & Taylor-net & HNN & ODE-net\\
  \hline
  Utilize physics prior&$\checkmark$&$\checkmark$&Partially\\
  Preserve symplectic structure&$\checkmark$&Partially&\\
  No need for intermediate data&$\checkmark$&&\\
  No need for analytical solution of derivative&$\checkmark$&&\\
  \tnote{1}Number of epochs until $L_{train}$ converges&100&1000&7000\\
  \tnote{2}Sample size needed for $L_{validation}\sim 10^{-4}$&15&50&50\\
  \hline
  \end{tabular}
  \begin{tablenotes}
        \footnotesize
        \item[1] In the pendulum problem with sample size 15.
        \item[2] In the pendulum problem, train each model until convergence.
      \end{tablenotes}
  }
  \label{tab:comparison}
  \end{threeparttable}
\end{table}

\section{High-dimensional systems}\label{sec:nbody}
We want to extend our model into higher-dimensional dynamical systems. Let's consider a more complicated system, a multidimensional N-body system. Its Hamiltonian is given by

 \begin{equation}
  \mathcal {H}(\bm{q}, \bm{p}) = \frac{1}{2}\sum_{i=1}^{N_{body}}\Vert\bm{p_i}\Vert^2-\sum_{1\leq i < j \leq N_{body}}\frac{1}{\Vert\bm{q_j}- \bm{q_i}\Vert},
  \label{eq:kep2}
 \end{equation}
where $N_{body}$ is the number of bodies in the system, and \[N_{body}\times(\textrm{dimension of space}) = N.\]

In a two-dimensional space, consider a system with $N_{body}>2$ bodies. The cost of collecting training data from all $N_{body}$ bodies may be high, and the training process may be time inefficient. Thus, instead of collecting information from all $N_{body}$ bodies to train our model, we only use data collected from two bodies as training data to make a prediction of the dynamics of $N_{body}$ bodies. This is based on the assumption that the interactive models between particle pairs with unit particle strengths $m=1$ are the same, and their corresponding Hamiltonian can be represented as network $\hat{\mathcal H}_{\theta}(\bm x_j,\bm x_k)$, based on which the corresponding Hamiltonian of $N_{body}$ particles can be written as \cite{battaglia2016interaction,sanchez2019hamiltonian}

    \begin{equation}
        \mathcal{{H}_{\theta}} = \sum_{i, j = 1}^{N_{body}} m_j m_k \hat{\mathcal H}_{\theta}(\bm{x_j},\bm{x_k}).
    \label{eq:nbody}
    \end{equation}
We embed \eqref{eq:nbody} into the symplectic integrator that includes $m_j$ to obtain the final network architecture.

The setup of N-body problem is similar to the previous problems. The training period is $T_{train} = 0.08$ and the prediction period is $T_{predict}=2\pi$. Similar to the setup of previous problems, the learning rate is decaying every 10 epochs. Learning rate, $\gamma$, $i$, $step\_size$, and $M$ are the same as the setup of Kepler problem in Table \ref{tab:problems}, except we use 40 samples to train our model. The training process takes about 100 epochs for the loss to converge.
In figure \ref{fig:nbodies}, we use our trained model to predict the dynamics of a 3-body system and a 6-body system.
In both cases, our model can predict the paths accurately, with the predicted paths in the 3-body system matching the true paths perfectly. The success of these tasks shows the strong generalization ability of our model. Based on our experiments, our model can be applied to problems with a larger scale, for example, to predict the motions of hundreds of bodies.

 \begin{figure}
  \centering
  \psfrag{m}{\scriptsize True Path}
  \psfrag{n}{\scriptsize Taylor-net Path, body1}
  \psfrag{l}{\scriptsize  Taylor-net Path, body2}
  \psfrag{o}{\scriptsize  Taylor-net Path, body3}
  \psfrag{p}{\scriptsize  Taylor-net Path, body4}
  \psfrag{q}{\scriptsize  Taylor-net Path, body5}
  \psfrag{r}{\scriptsize  Taylor-net Path, body6}
  \psfrag{a}[c][c]{\footnotesize (a)}
  \psfrag{b}[c][c]{\footnotesize (b)}
  \psfrag{x}[c][c]{\footnotesize $\bm q$}
  \psfrag{y}[c][c]{\footnotesize $\bm p$}
  \includegraphics[width=.48\linewidth]{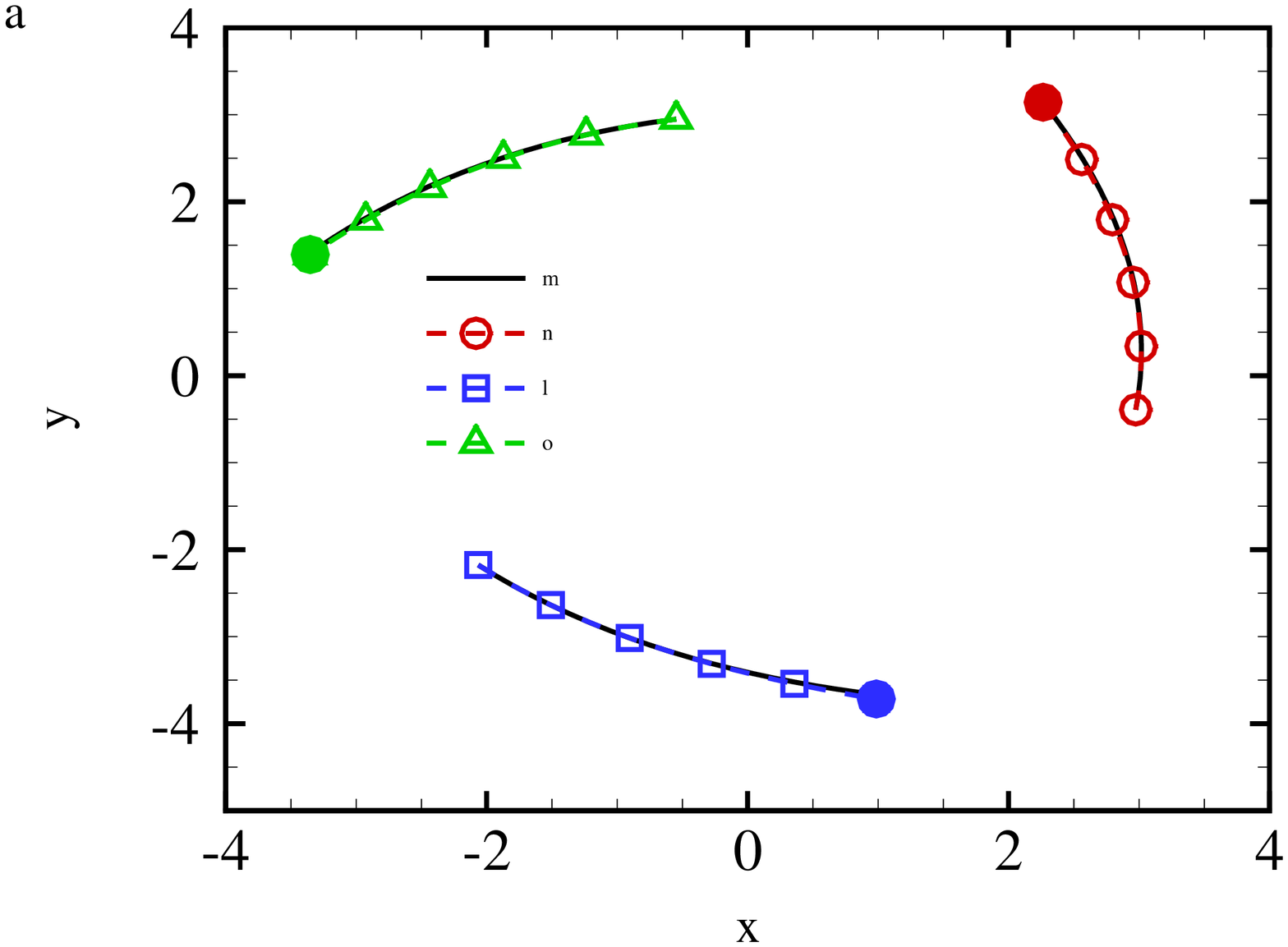}\hfill
  \includegraphics[width=.48\linewidth]{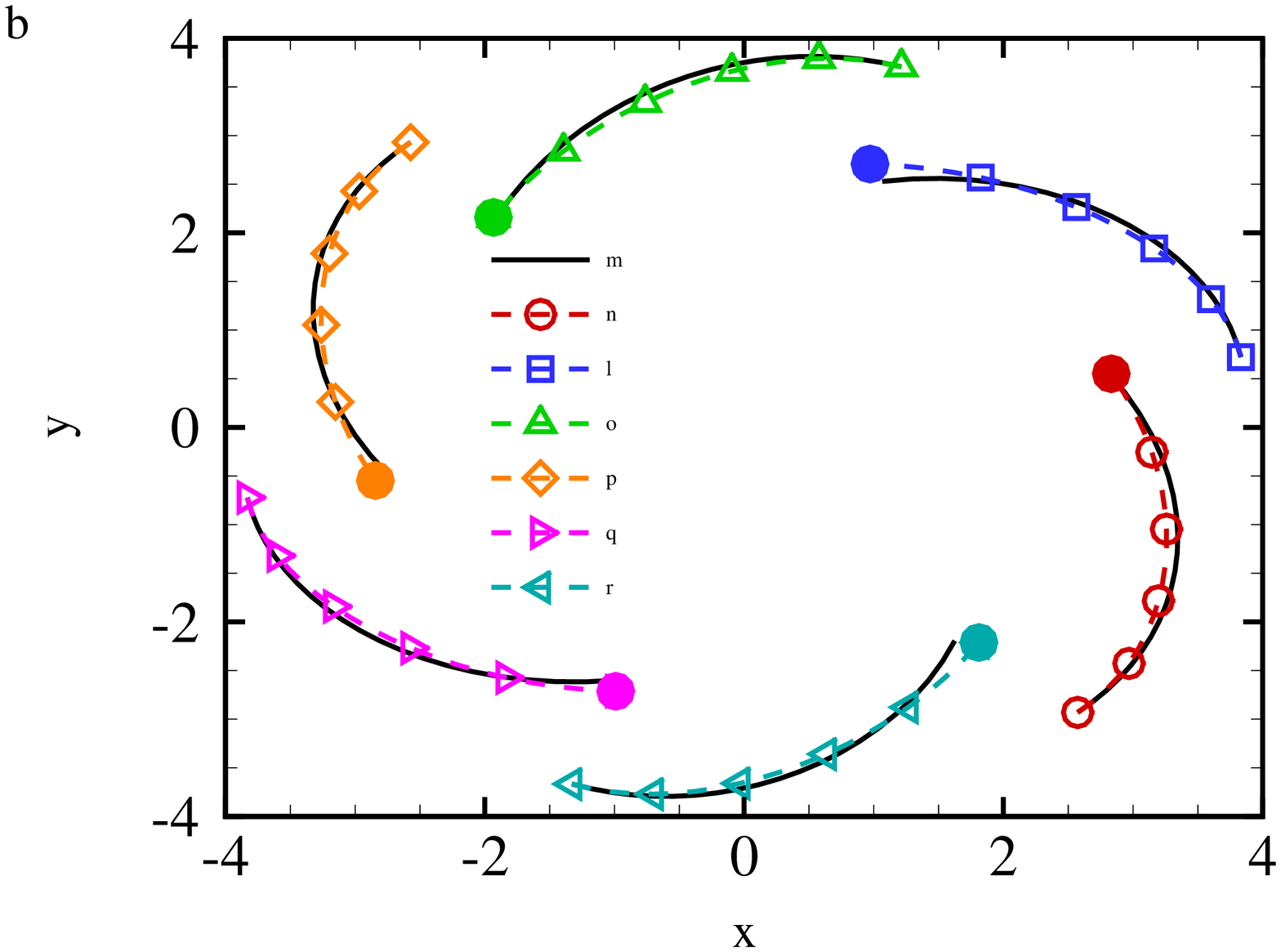}\\[0.5mm]
  \caption{Predicted position $\bm q$ and momentum $\bm p$ from $t=0$ to $t=2\pi$ (a) for 3 bodies and (b) for 6 bodies. In both (a) and (b), the training period is $T_{train} = 0.08$, and the prediction period is $T_{predict}=2\pi$. We use the same trained model to make the predictions in (a) and (b), which is trained for 100 epochs. }
  \label{fig:nbodies}
\end{figure}

\section{Conclusion}\label{sec:conclusions}

We present Taylor-nets, a novel neural network architecture that can conduct continuous, long-term predictions based on sparse, short-term observations. Taylor-nets consist of two sub-networks, whose outputs are combined using a fourth-order symplectic. Both sub-networks are embedded with the form of Taylor series expansion where each term is designed as a symmetric structure. Our model is able to learn the continuous-time evolution of the target systems while simultaneously preserving their symplectic structures. We demonstrate the efficacy of our Taylor-net in predicting a broad spectrum of Hamiltonian dynamic systems, including the pendulum, the Lotka--Volterra, the Kepler, and the H\'enon--Heiles systems.

We evaluate the performance of using the Taylor series as the underlying structure of Taylor-net by comparing it with the most used activation function, ReLU. The experimental results show that the neural networks perform better with Taylor series than with ReLU in the pendulum, the Lotka--Volterra, and the Kepler problems. In all three systems, the training loss of using the Taylor series is 10 to 100 times smaller than that of using ReLU. The strong representation ability of the Taylor series is an important factor that increases the accuracy of the prediction.

Moreover, we compare Taylor-net with other state-of-art methods, ODE-net and HNN, to access its predictive ability and robustness. We observe that the prediction error of Taylor-net over the prediction period is half of that of HNN and one-seventh of that of ODE-net. The predictions made by HNN and ODE-net also diverge from the true flow much faster as time increases. Additionally, to test the robustness of our model, we implement two testing cases with small and large noises. We add noise $\sigma_1, \sigma_2 \sim \mathcal{N}(0,0.1)$ in the case of small noise and $\sigma_1, \sigma_2 \sim \mathcal{N}(0,0.5)$ in the case of large noise. In the first case, Taylor-net discovers the unknown trajectory successfully, while ODE-net diverges away from the true value quickly. Although the predicting result of HNN does not seem to drift away from the true dynamics, it does not fit the true trajectory as well as the prediction made by Taylor-net. The prediction error of Taylor-net is about two-thirds and half of that of HNN and ODE-net respectively. The difference becomes clearer as we increase the noise. We observe that Taylor-net still makes predictions that are almost consistent with the true trajectories, while ODE-net completely fails to do so. Moreover, the prediction made by HNN is much worse than in the case of small noise, while the performance of Taylor-net remains as good as the previous case. The prediction error of Taylor-net is about half and one-twentieth of that of HNN and ODE-net respectively.

Additionally, we highlight the small training sample size and the fast convergence rate of our model. Under the same setting, HNN and OED-net need 5 times more samples than our model does to achieve the same validation loss, and their models take 10 times and 70 times more epochs to converge. We also test our model under only 1 epoch of training, the prediction results made by HNN and ODE-net completely fail to match the true flow, while Taylor-net predicts the truth to a level that is incomparable with HNN and ODE-net. Compared with HNN and OED-net, our model exhibits its unique computational merits by using small data with a short training period (6000 times shorter than the predicting period), small sample sizes, and no intermediate data to train the networks while outperforming others regrading the prediction accuracy, convergence rate, and robustness to a great extent.

Towards the end of our work in section \ref{sec:nbody}, we discussed the N-body system, which is a high-dimensional Hamiltonian system whose underlying governing equations are non-differentiable. In our future works, we will continue to explore solving this kind of high-dimensional problems, using some essential ideas of Taylor-nets with potential modifications. An other interesting direction will be to design a different neural network architecture with the same structure-preserving ability to learn the dynamics of non-separable Hamiltonian systems.

\section*{Acknowledgments}
This project is support in part by Neukom Institute CompX Faculty Grant, Burke Research Initiation Award, and NSF MRI 1919647. Yunjin Tong is supported by the Dartmouth Women in Science Project (WISP),  Undergraduate Advising and Research Program (UGAR), and Neukom Scholars Program.
Our code is available at \url{https://github.com/ytong6/Taylor-net}.

\bibliographystyle{plain}
\bibliography{refs}

\appendix
\section{Adjoint Method}\label{sec:adjoint}
\begin{figure}
    \centering
    \includegraphics[width=.9\linewidth]{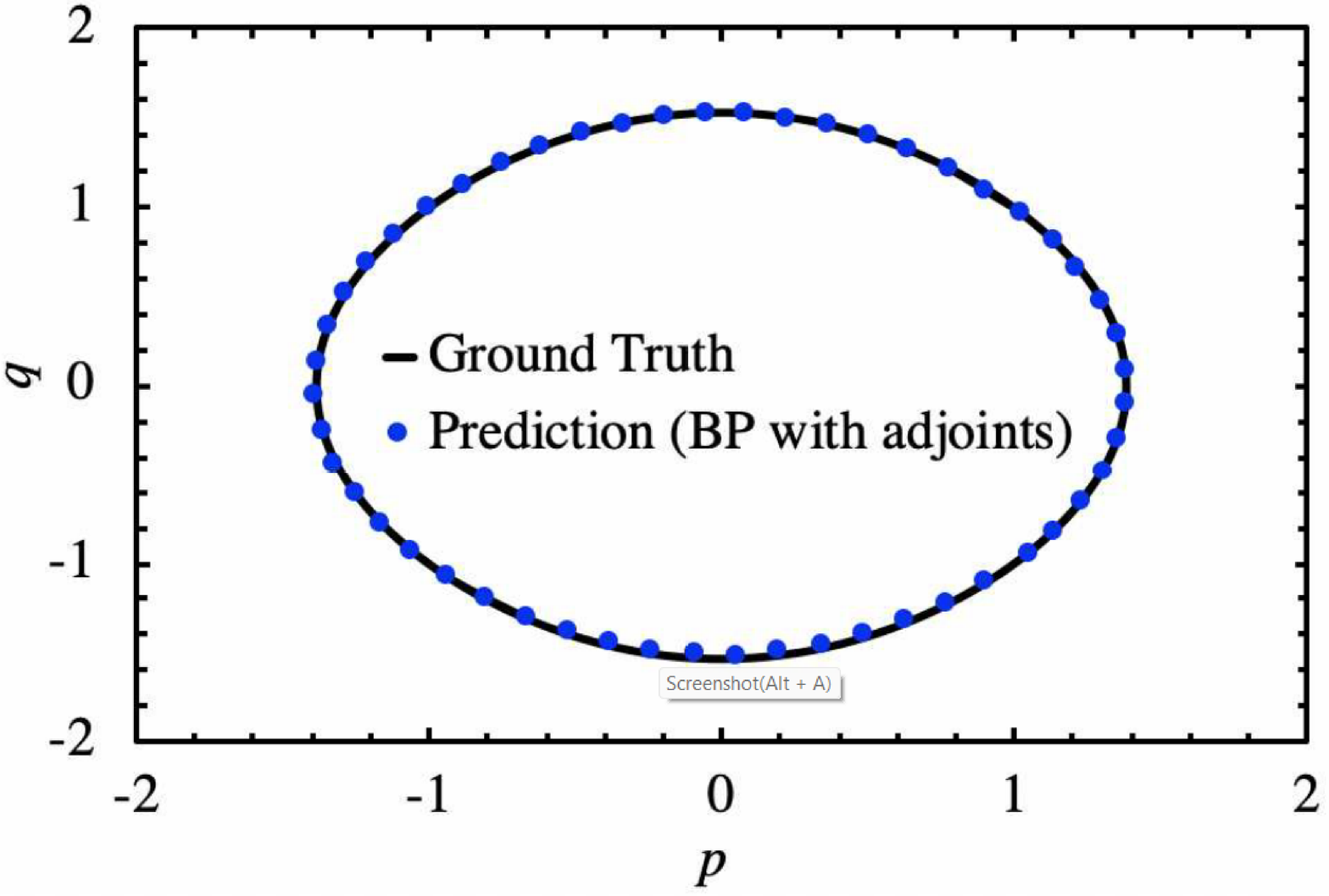}
	\caption{Prediction result of the pendulum problem using adjoint method as backward propagation}
	\label{fig:adjoint}
\end{figure}

Apply the chain rule to the gradients of loss function and consider the two neural networks $\bm T_p(\bm p,\bm \theta_p)$ and $\bm V_q(\bm q,\bm \theta_q)$ under the framework of neural ODEs, we obtain the following sets of equations:
 \begin{equation}
    \left\{
    \begin{aligned}
    \frac{\partial L}{\partial \bm \theta_p}&=\frac{\partial L}{\partial \bm q\left(t_{1}\right)} \frac{\mathrm{d} \bm q\left(t_{1}\right)}{\mathrm{d} \bm \theta_p} \\
    \bm q\left(t_{1}\right)&=\int_{t_0}^{t_1} \bm T_p(\bm p,\bm \theta_p) \mathrm{d} t+\bm q_{0}
    \end{aligned}\right.
  \end{equation}

  \begin{equation}
    \left\{
    \begin{aligned}
    \frac{\partial L}{\partial \bm \theta_q}&=\frac{\partial L}{\partial \bm p\left(t_{1}\right)} \frac{\mathrm{d} \bm p\left(t_{1}\right)}{\mathrm{d} \bm \theta_q} \\
   \bm p\left(t_{1}\right)&=-\int_{t_0}^{t_1} \bm V_q(\bm q,\bm \theta_q) \mathrm{d} t+ \bm p_{0}
    \end{aligned}\right.
  \end{equation}
 where $L$ is the loss function, and $\bm q\left(t_{1}\right)$, $\bm p\left(t_{1}\right)$, $\bm q_{0}$, and $\bm p_{0}$ are $\bm q$ and $\bm p$ at $t_1$ and $t_0$, respectively.

 Let $\bm b_p(t) =\mathrm{d} \bm q\left(t\right)/\mathrm{d} \bm \theta_p$ and $\bm b_q(t) =\mathrm{d} \bm p\left(t\right)/\mathrm{d} \bm \theta_q$, we derive the following equations:

\begin{equation}
 \begin{aligned}
  \bm b_p(t) =\frac{\mathrm{d} \bm q\left(t\right)}{\mathrm{d} \bm \theta_p}=\int_{t_0}^{t}\left[ \frac{\partial \bm T_p}{\partial \bm \theta_p}+\frac{\partial \bm T_p}{\partial \bm p} \bm b_p(\tau)\right] \mathrm{d} \tau\\
  \quad\Longrightarrow\quad
  \left\{
  \begin{aligned}
  \frac{\mathrm{d} \bm b_p(t)}{\mathrm{~d} t}&=\frac{\partial \bm T_p}{\partial \bm \theta_p}+\frac{\partial \bm T_p}{\partial \bm p} \bm b_p(t) \\
  \bm b_p(0)&=0,
  \end{aligned}\right.
  \end{aligned}
  \label{eq:bp}
\end{equation}

\begin{equation}
 \begin{aligned}
  \bm b_q(t) =\frac{\mathrm{d} \bm p\left(t\right)}{\mathrm{d} \bm \theta_q}=-\int_{t_0}^{t} \left[\frac{\partial \bm V_q}{\partial \bm \theta_q}+\frac{\partial \bm V_q}{\partial \bm q} \bm b_q(\tau)\right] \mathrm{d} \tau \\
  \quad\Longrightarrow\quad
  \left\{
  \begin{aligned}
  \frac{\mathrm{d} \bm b_q(t)}{\mathrm{~d} t}&=-\frac{\partial \bm V_q}{\partial \bm \theta_q}-\frac{\partial \bm V_q}{\partial \bm q} \bm b_q(t) \\
  \bm b_q(0)&=0.
  \end{aligned}\right.
  \end{aligned}
  \label{eq:bq}
\end{equation}

Given $\bm b_p$ and $\bm b_q$, we can rewrite the gradients of loss function as
\begin{equation}
 \frac{\partial L}{\partial \bm \theta_p}=\frac{\partial L}{\partial \bm q\left(t_{1}\right)} \bm b_p(t),
 \label{eq:lbp}
\end{equation}
and
\begin{equation}
 \frac{\partial L}{\partial \bm \theta_q}=\frac{\partial L}{\partial \bm p\left(t_{1}\right)} \bm b_q(t).
\label{eq:lbq}
\end{equation}
However, the scale for solving differential equations of $\bm b_q$ and $\bm b_p$ is too large. We therefore rewrite $\bm b_p$ and $\bm b_q$ as
\begin{equation}
\bm b_p(t)= \bm P_p(t) \int_{t_{0}}^{t} \bm P_p(\tau)^{-1} \frac{\partial \bm T_p}{\partial \bm \theta_p} \mathrm{d} \tau,
\end{equation}
and
\begin{equation}
\bm b_q(t)=- \bm P_q(t) \int_{t_{0}}^{t} \bm P_q(\tau)^{-1} \frac{\partial \bm V_q}{\partial \bm \theta_q} \mathrm{d} \tau.
\end{equation}

\begin{figure}
        \centering
        \includegraphics[width=.9\linewidth]{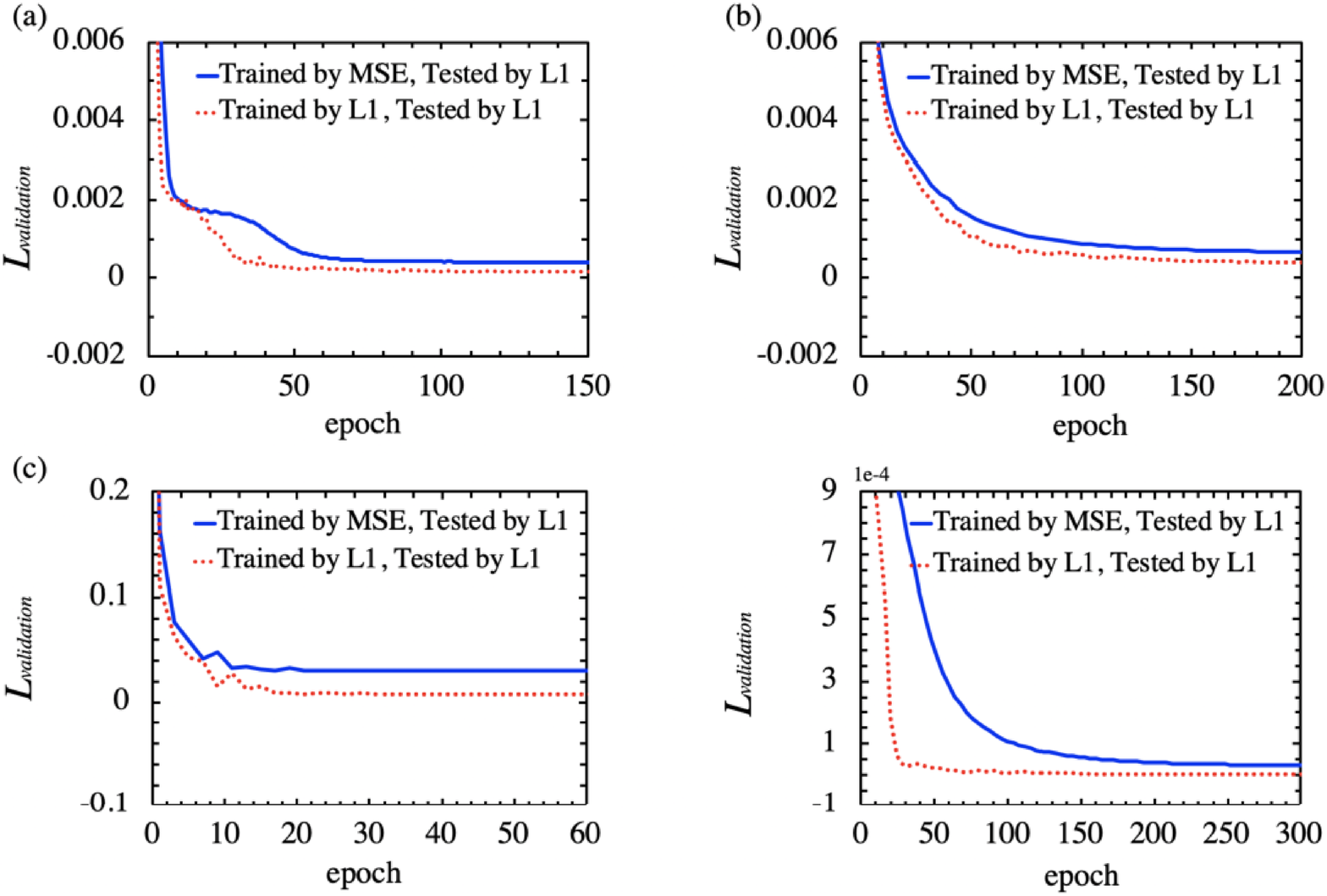}
    	\caption{Comparisons of validation losses with different training loss functions for (a) the pendulum, (b) the Lotka--Volterra, (c) the Kepler, and (d) the H\'enon--Heiles problems validated by L1 loss function. The red dashed lines represent the networks trained by L1 loss function; the blue solid lines represent the networks trained by MSE loss function.}
    	\label{fig:L1_comparison}
    \end{figure}
    
Substitute $\bm b_p$ and $\bm b_q$ into \eqref{eq:bp}, \eqref{eq:bq}, \eqref{eq:lbp}, and \eqref{eq:lbq}, we obtain two sets of equations:
\begin{equation}
    \left\{
    \begin{aligned}
  \frac{\partial L}{\partial \bm \theta_p}&=\frac{\partial L}{\partial \bm q(t_{1})} \bm P_p(t_{1})\int_{t_{0}}^{t_{1}} \bm P_p(t)^{-1} \frac{\partial \bm T_p}{\partial \bm \theta_p} \mathrm{d} t,\\
  \frac{\mathrm{d} \bm P_p(t)}{\mathrm{d} t}&= \frac{\partial \bm T_p}{\partial \bm p} \bm P_p(t),
    \end{aligned}\right.
\end{equation}
and
\begin{equation}
   \left\{
    \begin{aligned}
  \frac{\partial L}{\partial
  \bm \theta_q}&= - \frac{\partial L}{\partial \bm p(t_{1})} \bm P_q(t_{1})\int_{t_{0}}^{t_{1}} \bm P_q(t)^{-1} \frac{\partial \bm V_q}{\partial \bm \theta_q} \mathrm{d} t,\\
  \frac{\mathrm{d}\bm P_q(t)}{\mathrm{d} t}&= -\frac{\partial \bm V_q}{\partial \bm q} \bm P_q(t).
   \end{aligned}\right.
\end{equation}
    \begin{figure}
        \centering
        \includegraphics[width=.9\linewidth]{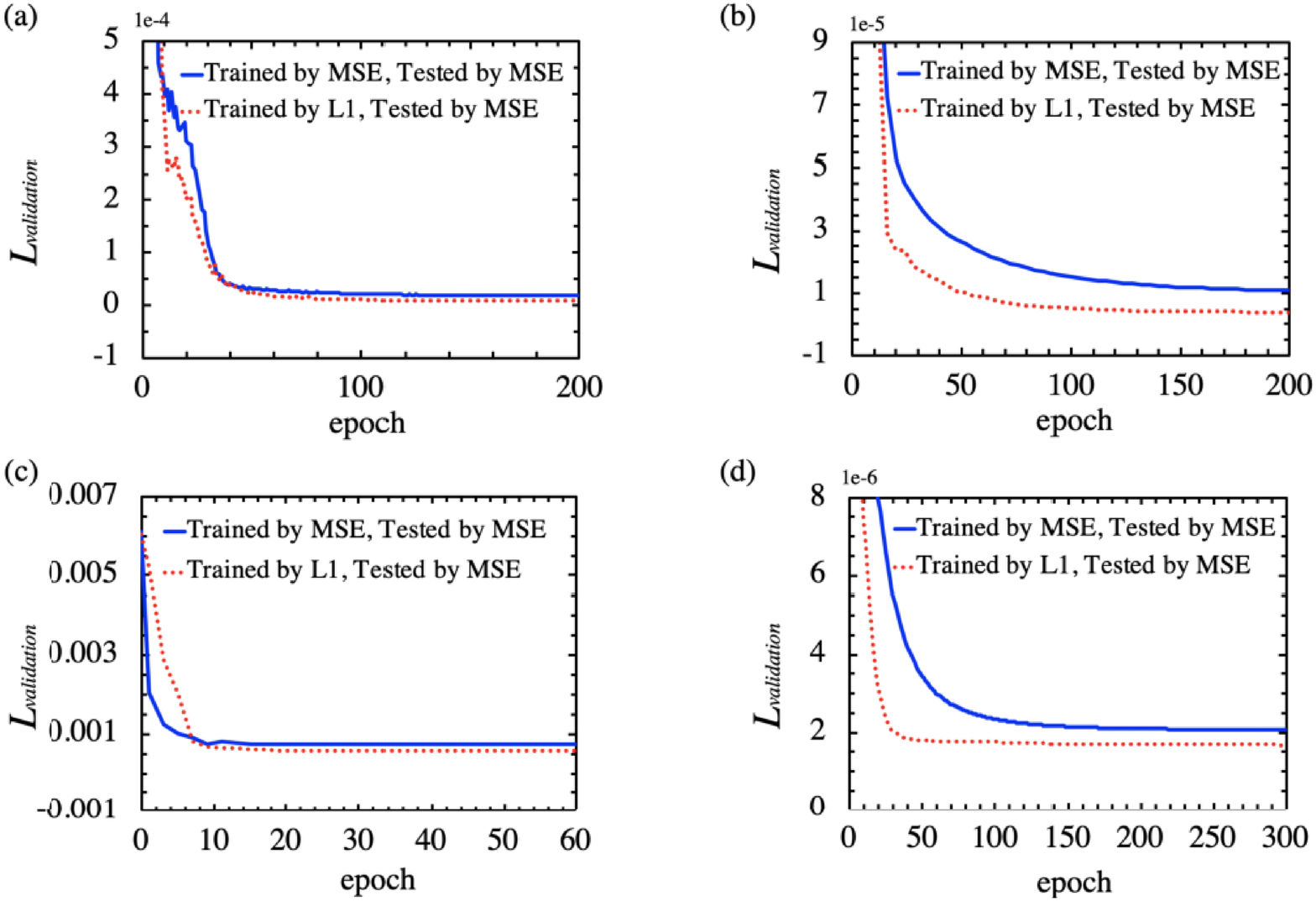}
    	\caption{Comparisons of validation losses with different training loss functions for (a) the pendulum, (b) the Lotka--Volterra, (c) the Kepler, and (d) the H\'enon--Heiles problems validated by MSE loss function. The red dashed lines represent the networks trained by L1 loss function; the blue solid lines represent the networks trained by MSE loss function.}
    	\label{fig:MSE_comparison}
    \end{figure}
However, the scale for solving $\bm P_p$ and $\bm P_q$ is still too large. We now consider the adjoint states $\bm a_p(t)$ and $\bm a_q(t)$:
\begin{equation}
  \bm a_p(t)=\frac{\partial L}{\partial \bm q(t_{1})} \bm P_p(t_{1}) \bm P_p(t)^{-1},\quad \bm a_p(t_1) = \frac{\partial L}{\partial \bm q(t_{1})},
  \label{eq:a_p}
\end{equation}
and
\begin{equation}
  \bm a_q(t)=\frac{\partial L}{\partial \bm p(t_{1})} \bm P_q(t_{1}) \bm P_q(t)^{-1},\quad \bm a_q(t_1) = \frac{\partial L}{\partial \bm p(t_{1})}.
\label{eq:a_q}
\end{equation}

We can then rewrite the gradient of loss function regarding to $\bm \theta_p$ as
\begin{equation}
  \frac{\mathrm{d} L}{\mathrm{d} \bm \theta_p}= \frac{\partial L}{\partial \bm q(t_{1})} \bm P_p(t_{1})\int_{t_{0}}^{t_{1}} \bm P_p(t)^{-1} \frac{\partial \bm T_p}{\partial \bm \theta_p} \mathrm{d} t=\int_{t_{0}}^{t_{1}} \bm a_p(t) \frac{\partial \bm T_p}{\partial \bm \theta_p} \mathrm{d} t.
  \label{eq:l_p}
\end{equation}
Similarly, the gradient of loss function regarding to $\bm \theta_q$ can be derived with the result differs by the sign
\begin{equation}
  \frac{\mathrm{d} L}{\mathrm{d} \bm \theta_q}= - \int_{t_{0}}^{t_{1}} \bm a_q(t)  \frac{\partial \bm V_q}{\partial \bm \theta_q} \mathrm{d} t.
\label{eq:l_q}
\end{equation}

We now want to derive the derivative of $\bm a_p(t)$ and $\bm a_q(t)$. The derivative of $\bm a_p(t)$ can be derived as follows
\begin{equation}
  \begin{aligned}
  \frac{\mathrm{d} \bm a_p}{\mathrm{d} t} &=\frac{\partial L}{\partial \bm q(t_{1})} \bm P_p(t_{1}) \frac{\mathrm{d} \bm P_p(t)^{-1}}{\mathrm{~d} t} \\
  &=-\frac{\partial L}{\partial \bm q\left(t_{1}\right)} \bm P_p\left(t_{1}\right) \bm P_p(t)^{-1} \frac{\mathrm{d} \bm P_p(t)}{\mathrm{d} t} \bm P_p(t)^{-1} \\
  &=-\bm a_p(t) \frac{\mathrm{d} \bm P_p(t)}{\mathrm{d} t} \bm P_p(t)^{-1} \\
  &=-\bm a_p(t) \frac{\partial \bm T_p}{\partial \bm p}.
  \end{aligned}
 \label{eq:ap_t}
 \end{equation}

The derivative of $\bm a_q(t)$ can be found in a similar manner. We obtain that

\begin{equation}
\frac{\mathrm{d} \bm a_q}{\mathrm{d} t} = \bm a_q(t) \frac{\partial \bm V_q}{\partial \bm q}.
\label{eq:aq_t}
\end{equation}

Combine the results we found in \eqref{eq:a_p}, \eqref{eq:a_q}, \eqref{eq:l_p}, \eqref{eq:l_q}, \eqref{eq:ap_t}, and \eqref{eq:aq_t}, we obtain the sets of equations that are our final result
\begin{equation}
  \begin{dcases}
  \frac{\partial L}{\partial \bm \theta_p}=\int_{t_{0}}^{t_{1}} \bm a_p(t) \frac{\partial \bm T_p}{\partial \bm \theta_p} \mathrm{d} t,\\
  \frac{\mathrm{d} \bm a_p}{\mathrm{d} t} =- \bm a_p(t) \frac{\partial \bm T_p}{\partial \bm p},\\
  \quad \bm a_p(t_1)= \frac{\partial L}{\partial \bm q(t_{1})},
  \end{dcases}
\label{L_thetap}
\end{equation}
\begin{equation}
  \begin{dcases}
  \frac{\partial L}{\partial \bm \theta_q}= - \int_{t_{0}}^{t_{1}} \bm a_q(t) \frac{\partial \bm V_q}{\partial \bm \theta_q} \mathrm{d} t,\\
  \frac{\mathrm{d} \bm a_q}{\mathrm{d} t} = \bm a_q(t) \frac{\partial \bm V_q}{\partial q},\\
  \quad \bm a_q(t_1)= \frac{\partial L}{\partial \bm p(t_{1})}.
  \end{dcases}
\label{L_thetaq}
\end{equation}
Using \eqref{L_thetap} and \eqref{L_thetaq}, we calculate the gradients of loss function in the backward propagation.

Figure \ref{fig:adjoint}, shows the prediction result of the pendulum problem using the adjoint method as our backward propagation method. We can see that the prediction result matches the ground truth well. However, training using the adjoint sensitivity method is about 30 percent slower than training using the automatic differentiation method due to higher time complexity.

\section{Loss function ablation test}\label{sec:ablation}

We conduct the ablation test on the pendulum, the
    Lotka--Volterra, the Kepler, and the H\'enon--Heiles problems to compare the validation loss after convergence with different training loss functions in the training process. Figure \ref{fig:L1_comparison} shows the comparison of validation losses with different training loss functions in the training process of different problems validated by L1 loss function. Figure \ref{fig:MSE_comparison} shows the comparison of validation losses with different training loss functions in the training process of different problems validated by MSE loss function. We observe that for all problems, the validation loss with $L1$ is smaller than that with MSE after convergence. The better performance of $L1$ may be due to MSE loss's high sensitivity to outliers. This explains why we choose $L1$ loss function as our training loss function.

\end{document}